\documentclass{article}
\usepackage[utf8]{inputenc}

\usepackage{microtype}
\usepackage{graphicx}
\usepackage{subcaption}
\usepackage{booktabs} 
\usepackage{setspace}

\usepackage{hyperref}

\usepackage[accepted]{misc/icml2023_dp4ml}

\usepackage{amsmath}
\usepackage{amssymb}
\usepackage{mathtools}
\usepackage{amsthm}
\usepackage{bm}
\usepackage{amsthm}
\usepackage{bbm}
\usepackage{mathrsfs}
\usepackage{svg} 
\usepackage{amssymb}
\usepackage{enumitem}
\usepackage{multirow}
\usepackage{tabularx}
\usepackage{wrapfig}
\usepackage{amsmath}
\usepackage{arydshln}
\usepackage{empheq}
\usepackage{import}
\usepackage{flafter}

\newtheorem{proposition}{Proposition}
\newtheorem{lemma}{Lemma}
\newtheorem{definition}{Definition}

\usepackage{mathtools}

\graphicspath{img}

\newcommand{\PkW}[0]{\mathbb{P}_{\mathcal{H}_{\mathcal{E}}}}
\newcommand{\WaW}[0]{\bm{W}^\ast\circ\bm{W}}
\newcommand{\WWa}[0]{\bm{W}\circ\bm{W}^\ast}
\newcommand{\Id}[1]{\bm{I}_{\mathcal{#1}}}
\newcommand{\Pm}[0]{\bm{\Phi}_{c}}
\newcommand{\Pma}[0]{\bm{\Phi}_{c}^\ast}
\newcommand{\PPa}[0]{\bm{\Phi}\circ\bm{\Phi}^\ast}

\newcommand{\PPam}[0]{\bm{\Phi}_{c} \circ \bm{\Phi}_{c}^\ast}
\newcommand{\PaPm}[0]{\bm{\Phi}_{c}^\ast \circ \bm{\Phi}_{c}}

\usepackage[capitalize,noabbrev]{cleveref}

\DeclareUnicodeCharacter{2212}{-}

\theoremstyle{plain}

\theoremstyle{definition}

\theoremstyle{remark}

\usepackage[textsize=tiny]{todonotes}

\icmltitlerunning{Duality for Probabilistic PCA}

\begin{document}

\twocolumn[
\icmltitle{A Dual Formulation for Probabilistic Principal Component Analysis}



\icmlsetsymbol{equal}{*}

\begin{icmlauthorlist}
\icmlauthor{Henri De Plaen}{stadius}
\icmlauthor{Johan A. K. Suykens}{stadius}
\end{icmlauthorlist}

\icmlaffiliation{stadius}{KU Leuven, Department of Electrical Engineering (ESAT), STADIUS Center for Dynamical Systems, Signal Processing and Data Analytics, Kasteelpark Arenberg 10, 3001 Leuven, Belgium}

\icmlcorrespondingauthor{Henri De Plaen}{henri.deplaen@esat.kuleuven.be}

\icmlkeywords{Machine Learning, Optimal Transport, Wasserstein, PCA, Principal Component Analysis, Kernel Methods, Restricted Kernel Machines}

\vskip 0.3in
]



\printAffiliationsAndNotice{}  

\setcounter{footnote}{1}
\begin{abstract}
In this paper, we characterize \emph{Probabilistic Principal Component Analysis} in Hilbert spaces and demonstrate how the optimal solution admits a representation in dual space. This allows us to develop a generative framework for kernel methods. Furthermore, we show how it englobes \emph{Kernel Principal Component Analysis} and illustrate its working on a toy and a real dataset.
\end{abstract}
\section{Introduction}
Classical datasets often consist of many features, making dimensionality reduction methods particularly appealing. \emph{Principal Component Analysis} (PCA) is one of the most straightforward frameworks to that goal and it is hard to find a domain in machine learning or statistics where it has not proven to be useful. PCA considers new decorrelated features by computing the eigendecomposition of the covariance matrix.

Probabilistic models on another side participate to the building of a stronger foundation for machine learning models. By considering models as probability distributions, we are able to natively access notions such as variance or sampling, \emph{i.e.} generation.
A probabilistic approach to PCA, known as \emph{Probabilistic Principal Component Analysis} (Prob. PCA), has been formulated by~\cite{ppca}. Its principles can be visualized in the primal part of Table~\ref{tab:interpretation}.  

Even endowed with a probabilistic interpretation, PCA remains restricted to linear relations between the different features. \emph{Kernel Principal Component Analysis} (KPCA)~\cite{mika1998kernel,scholkopf1998nonlinear} was an attempt to give a non-linear extension to (non-probabilistic) PCA by decomposing a kernel matrix instead of the covariance matrix. 
An earlier attempt to give a probabilistic formulation of KPCA has been done by~\cite{pkpca}. As developed further, the latter model does not consist in a kernel equivalent of the Prob. PCA, but rather in another model based on similar principles.

More recently, \emph{Restricted Kernel Machines}~\cite{rkm} opened a new door for a probabilistic version of PCA both in primal and dual. They essentially use the Fenchel-Young inequality on a variational formulation of KPCA~\cite{lssvm-pca,convex_pca} to obtain an energy function, closely resembling to \emph{Restricted Boltzmann Machines}. The framework has been further extended to generation~\cite{schreurs,winant2020latent}, incorporating robustness~\cite{pandey2020robust}, multi-view models~\cite{pandey2021generative}, deep explicit feature maps~\cite{pandey2022disentangled} or times-series~\cite{pandey2022recurrent}. 

\subsection{Contributions}
\begin{enumerate}
\itemsep0em
    \item We characterize the Prob. PCA framework in Hilbert spaces and give a dual interpretation to the model.
    \item We develop a new extension of KPCA incorporating a noise assumption on the explicit feature map.
    \item We give a probabilistic interpretation of the generation in KPCA.
    \item We illustrate how the dual model works on a toy and a real dataset and show its connections to KPCA\footnote{Resources: \url{https://hdeplaen.github.io/kppca}.}.
\end{enumerate}
\vspace{-.7em}
\begin{figure*}[ht!]
    \centering
    \def\svgwidth{\textwidth}
    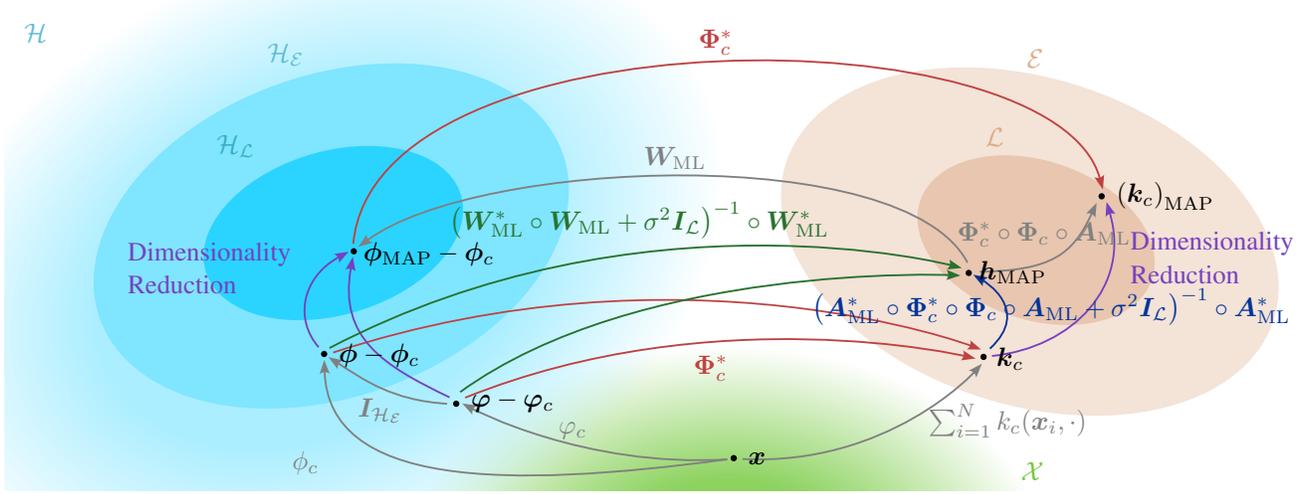
    \caption{Global overview of the Probabilistic Principal Component Analysis in both primal and dual formulations. The primal spaces, or feature $\mathcal{H}$, $\mathcal{H}_{\mathcal{E}}$ and $\mathcal{H}_{\mathcal{L}}$ are in blue. The dual, or kernel and latent spaces $\mathcal{E}$ and $\mathcal{L}$ are in brown. The input space $\mathcal{X}$ is in green. The color or the applications (arrows) is just for the readability and has nothing to do with the color of the spaces.}
    \label{fig:overview}
\end{figure*}

\begin{table*}[ht]
    \centering
    \resizebox{\textwidth}{!}{
    \begin{tabular}{llll}
        \textbf{Distribution} & \textbf{Interpretation} & \textbf{Primal (features)} & \textbf{Dual (kernels)}\\
        latent $|$ observation %
        & latent projection %
        & $\bm{h}|\bm{\phi} \sim \mathcal{N}\bigl(\bm{\Sigma}_{\bm{h}|\bm{\phi}}^{-1}\circ\bm{W}_{\mathrm{ML}}^\ast(\bm{\phi}-\bm{\phi}_c),\sigma^2\bm{\Sigma}_{\bm{h}|\bm{\phi}}^{-1}\bigr)$ %
        & $\bm{h}|\bm{k}_c \sim \mathcal{N}\bigl(
         \bm{\Sigma}_{\bm{h}|\bm{k}_c}^{-1}\circ \bm{A}_{\mathrm{ML}}\bm{k}_c,\bm{\Sigma}^{-1}_{\bm{h}|\bm{k}_c}\bigr)$ %
         \\
        observation $|$ latent %
        & latent-based generation %
        & $\bm{\phi} | \bm{h} \sim \mathcal{N}\bigl(\bm{W}_{\mathrm{ML}}\bm{h} - \bm{\phi}_c,\sigma^2\bm{I}_{\mathcal{H}_{\mathcal{E}}}\bigr)$%
        & $\bm{k}_c|\bm{h} \sim \mathcal{N}\bigl(( \PaPm)\circ \bm{A}_{\mathrm{ML}}\bm{h},\sigma^2\PaPm\bigr)$%
        \\
        latent %
        & latent prior %
        & $\bm{h} \sim \mathcal{N}(\bm{0},\bm{I}_{\mathcal{L}})$ %
        & $\bm{h} \sim \mathcal{N}(\bm{0},\bm{I}_{\mathcal{L}})$%
        \\
        observation %
        & absolute generation %
        & $\bm{\phi} \sim \mathcal{N}(\bm{\mu},\bm{W}_{\mathrm{ML}}\circ\bm{W}_{\mathrm{ML}}^\ast + \sigma^2\bm{I}_{\mathcal{H}_{\mathcal{E}}})$ %
        & $\bm{k}_c \sim \mathcal{N}\bigl(\bm{0},\bm{A}_{\mathrm{ML}}^\ast \circ \bm{A}_{\mathrm{ML}} +\sigma^2\left(\Pma\circ\Pm\right)^{-1}\bigr)$
    \end{tabular}}
    \caption{Interpretation of the different distributions of the Prob. PCA framework after training, in both primal and dual formulations. The covariance operators are given by $\bm{\Sigma}_{\bm{h}|\bm{\phi}} = (\bm{W}_{\mathrm{ML}}^\ast \circ \bm{W}_{\mathrm{ML}} + \sigma^2\Id{\mathcal{L}})^{-1}$ and $\bm{\Sigma}_{\bm{h}|\bm{k}_c} = \bigl(\bm{A}_{\mathrm{ML}}^\ast \circ (\PaPm) \circ \bm{A}_{\mathrm{ML}} + \sigma^2 \bm{I}_{\mathcal{L}}\bigr)^{-1}$, with maximum likelihood estimators for the primal and dual interconnection operators $\bm{W}_{\mathrm{ML}}$ and $\bm{A}_{\mathrm{ML}}$.}
    \label{tab:interpretation}
\end{table*}
\section{Primal and Dual Spaces}

The key idea behind the duality in PCA is that outer and inner products share the same eigenvalues. The consequence is that instead of decomposing the covariance matrix of any given feature map, we can decompose the associated Gram matrix, \emph{i.e.} the kernel matrix. The former is considered as the \emph{primal} formulation and the latter as the \emph{dual} formulation and they are both equivalent. Extending Prob. PCA to a dual formulation is however not straightforward: if all feature maps have an associated kernel, the converse is trickier. Some kernels correspond to feature maps in infinite dimensional spaces, where probability distributions cannot be properly defined. We therefore need to choose well defined finite subspaces to work in and consider linear operators instead of matrices. All formal definitions, propositions and proofs are provided in Appendix~\ref{app:theory}.

\subsection{Primal Spaces}
\textbf{Feature Space $\mathcal{H}$.} Given an input space $\mathcal{X}$, we first consider any feature map $\varphi : \mathcal{X} \rightarrow \mathcal{H}$. Following~\cite{convex_pca}, we will consider a separable, possibly infinite dimensional, Hilbert space $\left(\mathcal{H}, \langle\cdot,\cdot\rangle_{\mathcal{H}}\right)$. By $\bm{\varphi}$, we denote an element of $\mathcal{H}$ and its adjoint by $\bm{\varphi}^\ast = \langle \bm{\varphi}, \cdot \rangle \in \mathcal{H}^\ast$, with $\mathcal{H}^\ast \sim \mathcal{H}$ its Fréchet-Riesz dual space. Essentially, it corresponds to the transpose $\bm{\varphi}^\top$ in real, finite dimensional spaces as $\bm{\varphi}_1^\top\bm{\varphi}_2 = \langle\bm{\varphi}_1,\bm{\varphi}_2\rangle_{\mathcal{H}}$, but generalizes it for the possibly infinite dimensional spaces that will be necessary for the introduction of kernels. Furthermore, we assume our space to be defined over the reals such that $\langle\cdot,\cdot\rangle_{\mathcal{H}} : \mathcal{H} \times \mathcal{H} \rightarrow \mathbb{R}$ and its inner product is symmetric $\langle\bm{\varphi}_1,\bm{\varphi}_2\rangle_{\mathcal{H}} = \langle\bm{\varphi}_2,\bm{\varphi}_1\rangle_{\mathcal{H}}$. If $\mathcal{H}$ is of finite dimension $d$, we can therefore identify its canonical basis $\bm{u}_1,\ldots,\bm{u}_d$ with the canonical basis of $\mathbb{R}^d$. 

\textbf{Finite Feature Space $\mathcal{H}_{\mathcal{E}}$.} Considering a set of $N$ observations $\left\{\bm{x}_i \in \mathcal{X}\right\}_{i=1}^N$, the idea is to work directly in $\mathcal{H}$ by considering instead the feature map of the datapoints $\bm{\varphi}_i = \varphi\left(\bm{x}_i\right)$. We can however not define a normal distribution onto the full $\mathcal{H}$ yet as it is possibly infinite dimensional. We therefore have to consider a finite subspace $\mathcal{H}_{\mathcal{E}} \subset \mathcal{H}$. 
A natural choice would be $\mathcal{H}_{\mathcal{E}} = \mathrm{span}\left\{\bm{\varphi}_1,\ldots,\bm{\varphi}_N\right\}$. We now first have to find an orthonormal basis for $\mathcal{H}_{\mathcal{E}}$.

\subsection{Dual Spaces}
\textbf{Kernels.} For each feature map, there is an induced positive semi-definite kernel $k : \mathcal{X} \times \mathcal{X} \rightarrow \mathbb{R} : k\left(\bm{x},\bm{y}\right) = \langle \varphi(\bm{x}),\varphi(\bm{y})\rangle_{\mathcal{H}} = \varphi(\bm{x})^\ast\varphi(\bm{y})$. Inversely, to each positive semi-definite kernel corresponds a, possibly infinite dimensional, feature map, even if not explicitly defined. This follows from the theory of \emph{Reproducing Kernel Hilbert Spaces}. We refer to~\cite{kernels} for further info. 

\textbf{Kernel Space $\mathcal{E}$. }We now consider a finite dimensional Hilbert space $\left(\mathcal{E},\langle\cdot,\cdot\rangle_{\mathcal{E}}\right)$ of dimension $N$, the number of observations. 
It is defined similarly as above, with orthonormal basis $\bm{e}_1,\ldots,\bm{e}_N$. 
The basis also defines the identity over $\mathcal{E}$ as $\bm{I}_{\mathcal{E}} = \sum_{i=1}^N \bm{e}_i\bm{e}_i^\ast$. The goal for $\mathcal{E}$ is to represent the space of the kernel representations. We therefore define the linear operator $\bm{\Phi} : \mathcal{E} \rightarrow \mathcal{H} : \sum_{i=1} \bm{\varphi}_i\bm{e}_i^\ast$ and its adjoint $\bm{\Phi}^\ast : \mathcal{H} \rightarrow \mathcal{E} : \sum_{i=1}^N \bm{e}_i \bm{\varphi}_i^\ast$. Essentially, $\bm{\Phi}^\ast$ returns the kernel value with each datapoint: $\bm{\Phi}^\ast\varphi(\bm{x}) = \sum_{i=1}^N \bm{e}_i \left( \bm{\varphi}_i^\ast \varphi(\bm{x})\right) = \sum_{i=1}^N \bm{e}_i k\left( \bm{x}_i, \bm{x}\right)$ for any $\bm{x} \in \mathcal{X}$. Similarly, $\bm{\Phi}$ projects this value back as a linear combination of the different $\bm{\varphi}_i$'s, thus mapping back to $\mathcal{H}_{\mathcal{E}} \subset \mathcal{H}$. For this reason, the covariance $\bm{\Phi} \circ \bm{\Phi}^\ast = \sum_{i=1}^N \bm{\varphi}_i\bm{\varphi}_i^\ast$ acts as a projector from $\mathcal{H} \rightarrow \mathcal{H}_{\mathcal{E}}$. 
Its eigenvectors therefore form an orthonormal basis of the finite feature space $\mathcal{H}_{\mathcal{E}}$, which acts as the primal equivalent of the kernel space $\mathcal{E}$.

\textbf{Centered Kernels.} In most applications however, we prefer to work with the centered feature map, which we define as $\varphi_c(\cdot) = \varphi(\cdot) - \bm{\varphi}_c$ with $\bm{\varphi}_c = \frac1N \sum_{i=1}^N \bm{\varphi}_i$. We denote the associated kernel associated centered kernel $k_c : \mathcal{X} \times \mathcal{X} \rightarrow \mathbb{R} : k_c(\bm{x}_1,\bm{x}_2) = \varphi_c(\bm{x}_1)^\ast\varphi_c(\bm{x}_2)$. This leads to the definition of a new centered operator $\bm{\Phi}_c = \sum_{i=1} (\bm{\varphi}_i-\bm{\varphi}_c)\bm{e}_i^\ast = \bm{\Phi}\left(\bm{I}_{\mathcal{E}} - \frac1N \bm{1}_{\mathcal{E}\times\mathcal{E}}\right)$, with $\bm{1}_{\mathcal{E}\times\mathcal{E}} = \sum_{i,j=1}^N \bm{e}_i\bm{e}_{j}^\ast$. As always, we also consider its adjoint $\bm{\Phi}_{c}^\ast$. Considering the dual operator, we have $\PaPm = \sum_{i=1}^N (\bm{\varphi}_i-\bm{\varphi}_c)^\ast(\bm{\varphi}_i-\bm{\varphi}_c)\bm{e}_i\bm{e}_j^\ast = \sum_{i=1}^N k_c(\bm{x}_i,\bm{x}_j)\bm{e}_i\bm{e}_j^\ast$. We notice now that $\mathcal{H}_{\mathcal{E}} = \mathrm{span}\{\bm{\varphi}_1,\ldots,\bm{\varphi}_N\} = \mathrm{span}\{\bm{\varphi}_1-\bm{\varphi}_c,\ldots,\bm{\varphi}_N-\bm{\varphi}_c\}$ because $\bm{\varphi}_c$ is a linear combination of the elements of the basis. Therefore, the primal operator $\bm{\Phi}_c \circ \bm{\Phi}_c^\ast = \sum_{i=1}^N (\bm{\varphi}_i-\bm{\varphi}_c)(\bm{\varphi}_i-\bm{\varphi}_c)^\ast$ also acts as a projector from $\mathcal{H} \rightarrow \mathcal{H}_{\mathcal{E}}$ and we can choose its eigenvectors instead as an orthonormal basis of $\mathcal{H}_{\mathcal{E}}$. 

\textbf{Covariance and Kernels}. We now consider the key idea behind the duality in PCA: the operators $\PPam$ and $\PaPm$ are self-adjoint, positive semi-definite and share the same non-zero eigenvalues. We have $\bm{\Phi}_c \circ \bm{\Phi}_c^\ast = \sum_{i=1}^N \lambda_i\bm{v}_i\bm{v}_i^\ast$ and $\mathcal{H}_{\mathcal{E}} = \mathrm{span}\{\bm{v}_1,\ldots,\bm{v}_N\}$. Similarly, we have $\bm{\Phi}_c^\ast \circ \bm{\Phi}_c = \sum_{i=1}^N \lambda_i\bm{\epsilon}_i\bm{\epsilon}_i^\ast$ and $\mathcal{E} = \mathrm{span}\{\bm{\epsilon}_1,\ldots,\bm{\epsilon}_N\}$. The identity over the (primal) finite feature space $\mathcal{H}_{\mathcal{E}}$ can now be defined as $\bm{I}_{\mathcal{H}_{\mathcal{E}}} = \sum_{i=1}^N \bm{v}_i\bm{v}_i^\ast$ and the identity over the (dual) kernel space $\mathcal{E}$ as $\bm{I}_{\mathcal{E}} = \sum_{i=1}^N \bm{\epsilon}_i\bm{\epsilon}_i^\ast$. This is synthetized in the two first columns of Table~\ref{tab:spaces}. The identity over $\mathcal{H}$ reads $\bm{I}_{\mathcal{H}} = \bm{I}_{\mathcal{H}_{\mathcal{E}}} + \mathbb{P}_{\mathcal{H}_{\mathcal{E}}^\perp}$, with $\mathbb{P}_{\mathcal{H}_{\mathcal{E}}^\perp}$ a projector over the null space of $\bm{\Phi}_c \circ \bm{\Phi}_c^\ast$. It most be noted that it may happen that these basis may contain too much basis vectors if the two operators $\PaPm$ and $\PPam$ are not of full rank. In particular, this is the case when $\mathrm{dim}(\mathcal{H}) = d$ is finite and $d < N$. In this particular case, we would also have $\mathrm{dim}(\mathcal{H}_{\mathcal{E}}) = \mathrm{dim}(\mathcal{E}) = d$. Without loss of generality, we will assume that this is not the case. Similarly, we will neglect the case $N > d$ as we could just neglect the null space of $\PaPm$.





\textbf{Notations}. We can now define our probabilistic model over $\mathcal{H}_{\mathcal{E}}$. We will therefore use the notation $\phi$ instead of $\varphi$ to consider the feature map in our finite dimensional subspace $\mathcal{H}_{\mathcal{E}}$. More formally, we have $\phi : \mathcal{X} \rightarrow \mathcal{H}_{\mathcal{E}}:\bm{I}_{\mathcal{H}_{\mathcal{E}}}\circ \varphi$ and following from that  $\phi_c : \mathcal{X} \rightarrow \mathcal{H}_{\mathcal{E}}:\bm{I}_{\mathcal{H}_{\mathcal{E}}}\circ \varphi_c$.  In particular, we have the observations $\bm{\phi}_i = \phi(\bm{x}_i) = \bm{\varphi}_i$ and $\bm{\phi}_c = \bm{\varphi}_c$, as those are linear combinations of the basis. For the sake of readability, we will write $\bm{\phi} = \phi(\bm{x})$, the image of a random variable $\bm{x} \in \mathcal{X}$ and refer to it as a \emph{feature} observation or representation. 
Given any Hilbert space, $\bm{a}$ an element of it and a linear operator $\bm{\Sigma}$ from and to that space, we consider the \emph{multivariate normal distribution} $\bm{a} \sim \mathcal{N}\bigl(\bm{b},\bm{\Sigma}\bigr)$ as the distribution with density $\frac1Z\exp\bigl(-\frac12(\bm{a}-\bm{b})^\ast\bm{\Sigma}^{-1}(\bm{a}-\bm{b})\bigr)$. 
It is well defined if $Z$ is non-zero and finite.


\begin{table}[t]
    \centering




       
    \resizebox{\columnwidth}{!}{
    \begin{tabular}{llccc}
        \toprule
        \multicolumn{2}{l}{\textbf{Dimension}}
        & $d$ & $N$ & $q$ \\ \midrule

        \parbox[t]{2mm}{\multirow{3}{*}{\rotatebox[origin=c]{90}{Primal}}}
        & \textbf{Space}
        & $\mathcal{H} (= \mathbb{R}^d)$ & $\mathcal{H}_{\mathcal{E}}\subset \mathcal{H}$ & $\mathcal{H}_{\mathcal{L}}\subset \mathcal{H}_{\mathcal{E}}$ \\

        & \textbf{Canon. Basis}
        & $\{\bm{u}_i\}_{i=1}^d$ & N.A. & N.A. \\

        & \textbf{Other Basis}
        & N.A. & $\{\bm{v}_i\}_{i=1}^N$ & $\{\bm{\varrho}_p\}_{p=1}^q$ \\

        \midrule

        \parbox[t]{2mm}{\multirow{3}{*}{\rotatebox[origin=c]{90}{Dual}}}
        & \textbf{Space}
        & N.A. & $\mathcal{E} = \mathbb{R}^N$ & $\mathcal{L} \subset \mathcal{E}$ \\

        & \textbf{Canon. Basis}
        & N.A. & $\{\bm{e}_i\}_{i=1}^N$ & N.A. \\
        
        & \textbf{Other Basis}
        & N.A. & $\{\bm{\epsilon}_i\}_{i=1}^N$ & $\{\bm{r}_p\}_{p=1}^q$ \\
        \bottomrule
    \end{tabular}}
    \caption{Different spaces with their dimension and the canonical orthonormal basis if it applies, as well as another (non-canonical) basis when used throughout this paper. The equality $\mathcal{H} = \mathbb{R}^d$ only makes sense if $\mathrm{dim}(\mathcal{H}) =d$ is finite.}
    \label{tab:spaces}
\end{table}








\section{Primal Model}
We will now essentially follow the work of~\cite{ppca} and redefine the model distributions. This section corresponds to the primal formulation and we only consider the feature representations. It does not yet introduce the kernel representations, which will appear in the dual formulation (Section~\ref{sec:dual-model}).

\subsection{Model and Latent Space}
\textbf{Factor Analysis.} The starting point is to consider a \emph{factor analysis} relationship~\cite{bartholomew2011latent,basilevsky2009statistical} between the feature observations $\bm{\phi}$ and the latent variables $\bm{h}$. In particular, we consider
\begin{equation}
\label{eq:factor}
    \bm{\phi} = \bm{W}\bm{h} + \bm{\mu} + \bm{\zeta}.
\end{equation}


The observations $\bm{\phi}$ live in the primal space $\mathcal{H}_{\mathcal{E}}$ of dimension $N$. We consider an isotropic normal noise $\bm{\zeta} \sim \mathcal{N}\bigl(\bm{0},\sigma^2\bm{I}_{\mathcal{H}_{\mathcal{E}}}\bigr)$ of variance $\sigma^2 \in \mathbb{R}_{>0}$ and a mean $\bm{\mu} \in\mathcal{H}_{\mathcal{E}}$. 

\textbf{Latent Space $\mathcal{L}$.} The latent variables $\bm{h}$ on the other hand live in a latent dual space $\mathcal{L} \subset \mathcal{E}$ of dimension $q \leq N$. They are related by a primal \emph{interconnection linear operator} $\bm{W}$. As it was the case before with $\bm{\Phi}$, the interconnection operator does not project to the full space $\mathcal{H}_{\mathcal{E}}$ because of its reduced dimensionality. It therefore projects to yet another feature space $\mathcal{H}_{\mathcal{L}} \subset \mathcal{H}_{\mathcal{E}}$, which acts as the primal equivalent of the latent space $\mathcal{L}$. The equality of these two spaces only holds if $q = N$. We will therefore consider the mappings $\bm{W}^\ast : \mathcal{H}_{\mathcal{E}} \rightarrow \mathcal{L}$ and $\bm{W}: \mathcal{L} \rightarrow \mathcal{H}_{\mathcal{L}}$. The identity over $\mathcal{L}$ can be written as $\bm{I}_{\mathcal{L}} = \sum_{p=1}^q \bm{r}_p\bm{r}_p^\ast$, over $\mathcal{H}_{\mathcal{L}}$ as $\bm{I}_{\mathcal{H}_{\mathcal{L}}} = \sum_{p=1}^q \bm{\varrho}_p\bm{\varrho}_p^\ast$ and finally the identity over $\mathcal{H}_{\mathcal{E}}$ rewritten as $\bm{I}_{\mathcal{H}_{\mathcal{E}}} = \bm{I}_{\mathcal{H}_{\mathcal{L}}} + \mathbb{P}_{\mathcal{H}_{\mathcal{L}}^\perp}$, with $\mathbb{P}_{\mathcal{H}_{\mathcal{L}}^\perp}$ as a projector over the null space of $\bm{W}^\ast \circ \bm{W}$. This is summarized in the last column of Table~\ref{tab:spaces}.

\subsection{Feature Distributions}
\textbf{Latent-Based Generation.} The relation between the feature observations and the latent variables being set up (Eq.~\eqref{eq:factor}), we can derive the conditional probability of the feature observations given a latent variable:
\begin{equation}
\label{eq:conditional-feature}
\bm{\phi} | \bm{h} \sim \mathcal{N}\left(\bm{W}\bm{h} - \bm{\mu},\sigma^2\bm{I}_{\mathcal{H}_{\mathcal{E}}}\right).  
\end{equation}

As discussed earlier, we see that the latent variables do not participate to the full scope of the observations in $\mathcal{H}_{\mathcal{E}}$, but only to their component in $\mathcal{H}_\mathcal{L}$. The rest is only constituted from the isotropic normal noisy mean. This distribution can be interpreted as a generative one: given a latent variable, we can sample a variable in feature space.

\textbf{Absolute Generation}. Considering the latent prior $\bm{h} \sim \mathcal{N}\bigl(\bm{0},\bm{I}_{\mathcal{L}}\bigr)$, we can derive the marginal distribution of the observations in feature space:
\begin{equation}
\label{eq:data-distr}
    \bm{\phi} \sim \mathcal{N}\left(\bm{\mu},\bm{W}\circ\bm{W}^\ast + \sigma^2\bm{I}_{\mathcal{H}_{\mathcal{E}}}\right).
\end{equation}
It can be considered as the data distribution of the model. Sampling from it also means generating feature representations in a more absolute way, \emph{i.e.}, without considering any latent variable, or more precisely considering a random latent variable according to its prior. As a consequence of Eq.~\eqref{eq:conditional-feature} and the isotropic aspect of the latent prior, we see that the observations are only non-isotropically distributed in $\mathcal{H}_{\mathcal{L}}$. Again, the rest is only the isotropically normally noisy mean. In other words, this means that the model parameter $\bm{W}$ only influences $\bm{\phi}$ for its components in $\mathcal{H}_{\mathcal{L}}$. 

\subsection{Training the Model}
\textbf{Maximum Likelihood.} As we now have the marginal distribution of the model (Eq.~\eqref{eq:data-distr}), the goal is to find the optimal hyperparameters $\bm{W}$ and $\bm{\mu}$ to match the set of observations $\{\bm{\phi}_i\}_{i=1}^N$. One way to determine them is by maximizing the likelihood of our observations. 
The \emph{Maximum Likelihood} (ML) estimator for the hyperparameters is given by:
\begin{eqnarray}
    \bm{\mu}_{\mathrm{ML}} &=& \bm{\phi}_c,\\
    \bm{W}_{\mathrm{ML}} &=& \sum_{p=1}^q \sqrt{\lambda_p/N-\sigma^2}\bm{v}_p\bm{r}_p^\ast, \label{eq:w-ml}
\end{eqnarray}
with $\left\{\left(\lambda_p,\bm{v}_p\right)\right\}_{p=1}^q$ the $q$ dominant eigenpairs of $\PPam$ ($\lambda_1\geq\cdots\geq\lambda_q\geq\cdots\lambda_N$), and $\left\{\bm{r}_p\right\}_{p=1}^q$ and arbitrary orthonormal basis of the latent space $\mathcal{L}$. The choice for the latter basis is arbitrary and makes the model rotational invariant in latent space. An additional condition is that $\sigma^2 \leq \lambda_{q}/N$. 
It is not surprising to see that the optimal mean $\bm{\mu}_{\mathrm{ML}}$ corresponds to the mean of the observations $\bm{\phi}_c$. We observe that $\bm{W}_{\mathrm{ML}}$ corresponds to the eigendecomposition of the centered covariance, at the exception that the noise assumption is substracted from its spectrum. By looking back at Eq.~\eqref{eq:factor}, it makes sense to avoid the noise in $\bm{W}_{\mathrm{ML}}$ as it is still going to be added by the term $\bm{\zeta}$.

\textbf{Noise Variance.} Maximizing the likelihood 
as a function of $\sigma^2$ leads to
    \begin{equation}
    \label{eq:mle-sigma2}
        \sigma^2_{\mathrm{ML}} = \frac{1}{N(N-q)}\sum_{p=q+1}^N \lambda_p.
    \end{equation}
The eigenvalue $\lambda_p$ corresponds to the variance for each component $\bm{v}_p$ of the covariance $\PPam$. The total variance of the data, noise included, is equal to $\frac1N\sum_{p=1}^N \lambda_p$ and the variance learned by the model through the primal interconnection operator to $\frac1N\sum_{p=1}^q\lambda_p$. Hence, the maximum likelihood estimator for the noise variance $\sigma^2_{\mathrm{ML}}$ can be interpreted as the mean of the variance that is discarded by the model. It also verifies the earlier condition that $\sigma^2 \leq \lambda_q/N$, as the eigenvalues are taken in descending order. It can be interpreted as the normalized mean variance of the left over eigendirections, \emph{i.e.} the orthogonal space of the latent space: $\mathcal{L}^\perp = \mathcal{E}\backslash\mathcal{L}$. By consequence, we may decide to choose the latent dimension $q = \mathrm{dim}(\mathcal{L})$ and deduct $\sigma^2_{\mathrm{ML}}$. In the opposite, we may also decide to set an arbitrary $\sigma^2$ and deduct the latent dimension $q$ instead. We therefore can consider either $\sigma^2$ or $q$ as an additional hyperparameter. We must however keep in mind that this is strongly going to be influenced by the distribution of the eigenvalues and that the latent dimension $q$ for the same $\sigma^2$ may heavily vary from application to application. 

\textbf{Uncentered Features.} We may also consider not to consider the mean as an optimizable hyperparameter and set it arbitrarily to $\bm{\mu} = \bm{0}$. In this case, Eq.~\eqref{eq:w-ml} would be the same at the difference that the $\bm{W}_{\mathrm{ML}}$ would be constructed from the dominant eigenpairs of the uncentered covariance $\bm{\Phi} \circ \bm{\Phi}^\ast$ instead of its centered counterpart $\PPam$.


\subsection{Dimensionality Reduction in Feature Space}
\label{sec:dim-red-primal}
\textbf{Latent Projection.} Up to now, we only considered the distribution of the feature variables $\bm{\phi}$. We can also calculate the posterior distribution of the latent variable $\bm{h}$ given the primal feature variable $\bm{\bm{\phi}}$:
\begin{equation}
\label{eq:conditional-latent-primal}
    \bm{h} | \bm{\phi} \sim \mathcal{N}\left(\bm{\Sigma}_{\bm{h}|\bm{\phi}}^{-1} \circ \bm{W}^\ast(\bm{\phi}-\bm{\mu}),\sigma^2\bm{\Sigma}_{\bm{h}|\bm{\phi}}^{-1} \right),
\end{equation}
with $\bm{\Sigma}_{\bm{h}|\bm{\phi}} = \left(\WaW + \sigma^2\Id{\mathcal{L}}\right)^{-1}$. The mean of the distribution can be considered as a pseudo-inverse of the observation $\bm{\phi}$, but regularized by $\sigma^2$. This regularization ensures to avoid the noise. If the prior of the latent variables was isotropic, this is not the case anymore for the posterior. If we consider the maximum likelihood estimator for the primal interconnection operator $\bm{W}_{\mathrm{ML}}$, the variance becomes $\sigma^2\bm{\Sigma}_{\bm{h}|\bm{\phi}}^{-1} = N\sigma^2\sum_{p=1}^q \lambda_p^{-1} \bm{r}_p\bm{r}_p^\ast$. It can be interpreted as the uncertainty for each component of the latent variable $\bm{h}$ (w.r.t. the eigendirection $\bm{r}_p$), due to the noise assumption. By consequence, the greater the explained variance $\lambda_p$ for the eigendirection $\bm{v}_p$ of the covariance $\PPam$, the smaller the corresponding uncertainty on the component $\bm{r}_p$ of the latent vairable $\bm{h}$. For each observation in feature space $\bm{\phi}$, this returns a distribution for the latent variable $\bm{\phi}$ and can therefore be considered as a sort of probabilistic projection in latent space $\mathcal{L}$. 


\textbf{Maximum A Posteriori.} Up to now, we were only considering distributions. The only way to go from a feature representation to a latent variable or the opposite was probabilistic. In order to have a deterministic approach, we need proper mappings. One way is to consider the \emph{Maximum A Posteriori} (MAP) of $\bm{h}$ given $\bm{\phi}$. 
It maps the feature observation $\bm{\phi} \in \mathcal{H}_{\mathcal{E}}$ to latent variable $\bm{h}_{\mathrm{MAP}} \in \mathcal{L}$, hence reducing the dimensionality of any input to that of the latent space. To allow it to work for any input $\bm{\varphi} \in \mathcal{H}$, we may again consider the projection $\bm{\phi} = \bm{I}_{\mathcal{H}_{\mathcal{E}}} \bm{\varphi}$. As $\bm{W}_{\mathrm{ML}}^\ast \circ \bm{I}_{\mathcal{H}_{\mathcal{E}}} = \bm{W}_{\mathrm{ML}}^\ast$: 
\begin{equation}
\label{eq:primal-latent-map}
\begin{split}
    \bm{h}_{\mathrm{MAP}} =& \left( \bm{W}_{\mathrm{ML}}^\ast \circ \bm{W}_{\mathrm{ML}} + \sigma^2\bm{I}_{\mathcal{L}}\right)^{-1} \\
    &\circ \bm{W}_{\mathrm{ML}}^\ast \left(\bm{\varphi}-\bm{\varphi}_c\right).
\end{split}
\end{equation}
To map back to the feature space $\mathcal{H}_{\mathcal{L}}$, we may consider the \emph{maximum a posteriori} of $\bm{\phi}$ given $\bm{h}$ (Eq.~\eqref{eq:data-distr}). This gives
\begin{equation}
\label{eq:phi-map}
    \bm{\phi}_{\mathrm{MAP}} = \bm{W}_{\mathrm{MAP}} \bm{h} + \bm{\phi}_c.
\end{equation}

The final projection reads
\begin{equation}
\begin{split}
    \bm{\phi}_{\mathrm{MAP}} =& \bm{W}_{\mathrm{ML}}\circ \left( \bm{W}_{\mathrm{ML}}^\ast \circ \bm{W}_{\mathrm{ML}} + \sigma^2\bm{I}_{\mathcal{L}}\right)^{-1} \\
    &\circ \bm{W}_{\mathrm{ML}}^\ast\left(\bm{\varphi}-\bm{\varphi}_c\right) + \bm{\phi}_c.
\end{split}
\end{equation}

\textbf{No Noise.} We may also decide not to consider $\sigma^2$ as a parameter to optimize and set it to an arbitrary value. The latent dimensions $q$ could also be set an arbitrary value, without it to be related to the latent dimension $q$ according to Eq.~\eqref{eq:mle-sigma2}. We notice that in the limit of $\sigma^2 \rightarrow 0$, we recover the classical Principal Component Analysis reconstruction scheme. Indeed the conditional probability distributions become exact relations. We also notice that the condition $\sigma^2 \leq \lambda_q/N$ (Prop.~\ref{prop:mle-primal}) is then always satisfied. Furthermore, when $q = \mathrm{dim}(\mathcal{H}_{\mathcal{E}})$, the reconstruction is perfect in $\mathcal{H}_{\mathcal{E}}$ and in particular for our original observations $\{\bm{\varphi}_i\}_{i=1}^N$ and $\bm{\varphi}_c$ (as we have $\bm{\phi}_i = \bm{\varphi}_i$). Indeed, we would have
\begin{equation}
\label{eq:h-map}
    \bm{h}_{\mathrm{MAP}} = \bm{W}_{\mathrm{ML}}^+ \left(\bm{\varphi}-\bm{\varphi}_c\right),
\end{equation}
with $\bm{W}_{\mathrm{ML}}^+$ the Moore-Penrose pseudo-inverse of $\bm{W}_{\mathrm{ML}}$. . We note here the symmetry with Eq.~\eqref{eq:phi-map}. If the maximum likelihood estimator for $\sigma^2$ is to be respected (Eq.~\eqref{eq:mle-sigma2}), this would mean that all components are kept ($\mathcal{L} = \mathcal{E}$) and the model reconstructs the full feature variance. In this case, the primal interconnection operator would become $\bm{W}_{\mathrm{ML}} = \sum_{p=1}^N \sqrt{\lambda_p/N}\bm{v}_p\bm{r}_p^\ast$ and be invertible. Its Moore-Penrose pseudo-inverse would become an exact inverse. Eqs.~\eqref{eq:phi-map} and \eqref{eq:h-map} would become exact opposites and there would be no loss due to the dimensionality reduction as there would be no noise to discard. By consequence, the reduction would become an identity over $\mathcal{H}_{\mathcal{E}}$: $\bm{\phi}_{\mathrm{MAP}} - \bm{\phi}_c = \bm{I}_{\mathcal{H}_{\mathcal{L}}}\left(\bm{\varphi}-\bm{\varphi}_c\right)$.


\section{Dual Model}
\label{sec:dual-model}
\textbf{Kernels without Dual.} In~\cite{pkpca}, the authors made the kernel matrix appear by considering the new observations $\bigl\{ \sum_{i=1}^d\bm{u}_i\bm{u}_j^\ast\phi(\bm{x}_i)\bigr\}_{j=1}^N$. In other words, each new datapoint consists in one particular feature of the feature map, for each original datapoint. If the original datapoints were organized as a matrix in $\mathbb{R}^{N \times d}$, this would correspond to taking its transpose as new datapoints. 
The outer product of the covariance matrix is transformed to the inner product of the kernel matrix. If indeed this formulation makes the kernel appear, it is not a dual formulation of the original problem, but another problem. In this section, we show how the spaces defined hereabove help us build an equivalent dual formulation of the problem.

\textbf{Dual Formulation.} While keeping an equivalence with the primal model, we will now see that we can directly work in dual spaces $\mathcal{E}$ and $\mathcal{L}$ without considering the feature spaces at all, \emph{i.e.} resorting to the primal space $\mathcal{H}$ and its subsets. As we did for the primal feature variable $\bm{\phi}$, we will consider $\bm{k}_c = \Pma (\bm{\phi}-\bm{\phi}_c) = \sum_{i=1}^N k_c(\bm{x},\bm{x}_i)\bm{e}_i$ to represent the image in $\mathcal{E}$, of a random variable $\bm{x} \in \mathcal{X}$. We will refer to it as a \emph{dual feature variable}.
\subsection{Representation}
Considering the dual spaces, we can always express the interconnection operator $\bm{W}$ in the (non-orthonormal) basis $\left\{\bm{\phi}_1 - \bm{\phi}_c,\ldots,\bm{\phi}_N-\bm{\phi}_c\right\}$. As a consequence, we can always write
\begin{equation}
    \bm{W} = \bm{\Phi}_c \circ \bm{A},
\end{equation}
with $\bm{A} : \mathcal{L} \rightarrow \mathcal{L}$, the dual interconnection operator. Given the maximum likelihood estimator for the primal interconnection operator $\bm{W}_{\mathrm{ML}}$, we can directly deduce the dual one:
\begin{equation}
    \bm{A}_{\mathrm{ML}} = \sum_{p=1}^q \sqrt{1/N-\sigma^2\lambda_p^{-1}}\bm{\epsilon}_p\bm{r}_p^\ast,
\end{equation}
with $\left\{\left(\lambda_p,\bm{\epsilon}_p\right)\right\}_{p=1}^q$ the $q$ dominant eigenpairs of $\PaPm$ and $\left\{\bm{r}_p\right\}_{p=1}^q$ an arbitrary orthonormal basis of the latent space $\mathcal{L}$. The rotational invariance of the dual interconnection operator $\bm{A}_{\mathrm{ML}}$ is inherited from its primal counterpart $\bm{W}_{\mathrm{ML}}$. Again, if we consider an optimized mean $\bm{\mu} = \bm{0}$, we would have the relation $\bm{W}_{\mathrm{ML}} = \bm{\Phi} \circ \bm{A}_{\mathrm{ML}}$ with $\bm{A}_{\mathrm{ML}}$ then based on the eigenpairs of the non-centered $\bm{\Phi}^\ast \circ \bm{\Phi}$ instead. Using the same structure for $\bm{A}_{\mathrm{ML}}$, the optimal (primal) interconnection operator $\bm{W}_{\mathrm{ML}}$ could be expressed in the (non-ortonormal) basis $\{\bm{\phi}_1,\ldots,\bm{\phi}_N\}$.


\subsection{Kernel Distributions}
\textbf{Projection and Generation}. We can also consider the dual counterparts of the distributions of the primal model (Eqs.~\eqref{eq:conditional-feature} and \eqref{eq:conditional-latent-primal}). For the sake of simplicity and to avoid heavier equations with non-centered kernels, we will only consider here the equations of the trained model, in particular with $\bm{\mu}_{\mathrm{ML}} = \bm{\phi}_c$ leading to centered kernels:
\begin{eqnarray}
    \bm{k}_c|\bm{h} &\sim& \mathcal{N}\bigl(( \PaPm)\circ \bm{A}_{\mathrm{ML}}\bm{h},\sigma^2\PaPm\bigr),\\
    \bm{h}|\bm{k}_c &\sim& \mathcal{N}\left(\bm{\Sigma}_{\bm{h}|\bm{k}_c}^{-1}\circ \bm{A}_{\mathrm{ML}}\bm{k}_c,\bm{\Sigma}^{-1}_{\bm{h}|\bm{k}_c}\right),
\end{eqnarray}
 with $\bm{\Sigma}_{\bm{h}|\bm{k}_c} = \left(\bm{A}_{\mathrm{ML}}^\ast \circ \bigl(\PaPm\bigr) \circ \bm{A}_{\mathrm{ML}} + \sigma^2 \bm{I}_{\mathcal{L}}\right)^{-1}$.

 \subsection{Dimensionality Reduction in Kernel Space}
\textbf{Maximum A Posteriori.} This now allows us to consider the dimensionality reduction in kernel space in a similar fashion as in Section~\ref{sec:dim-red-primal}. Again we consider the MAP of the latent variable $\bm{h}$ given the kernel representation $\bm{k}_c$:
 \begin{equation}
 \begin{split}
     \bm{h}_{\mathrm{MAP}} =& \left(\bm{A}_{\mathrm{ML}}^\ast \circ \bigl(\PaPm\bigr) \circ \bm{A}_{\mathrm{ML}} + \sigma^2 \bm{I}_{\mathcal{L}}\right)^{-1} \\
     &\circ \bm{A}_{\mathrm{ML}}\bm{k}_c, 
 \end{split}
 \end{equation}
 and similarly with the MAP of the kernel representation $\bm{k}_c$ given the latent variable $\bm{h}$:
 \begin{equation}
     \left(\bm{k}_c\right)_{\mathrm{MAP}} = \left( \PaPm\right)\circ \bm{A}_{\mathrm{ML}}\bm{h}.
 \end{equation}
 As for the primal model, the dimensionality reduction in dual is computed as $\left(\bm{k}_c\right)_{\mathrm{MAP}} = \left( \PaPm\right)\circ \bm{A}_{\mathrm{ML}}\bm{h}_{\mathrm{MAP}}$. 
 
 \textbf{No Noise.} Again, considering $\sigma^2 \rightarrow 0$ makes both dual conditional distributions become exact relations. In a ML context for $\sigma^2$ (Eq.~\eqref{eq:mle-sigma2}), this would imply that $q = \mathrm{dim}(\mathcal{E})$ and we would recover an identity $\left(\bm{k}_c\right)_{\mathrm{MAP}} = \bm{k}_c$, \emph{i.e.} no reduction. Without considering a ML context for $\sigma^2 \rightarrow 0$ and choosing an arbitrary $q \leq \mathrm{dim}(\mathcal{E})$, the reduction become exactly the reconstruction done in KPCA.

\subsection{Kernel Sampling}


\textbf{Probabilistic Sampling.} The dual counterpart of Eq.~\eqref{eq:data-distr} after training is given by
\begin{equation}
    \bm{k}_c \sim \mathcal{N}\left(\bm{0},\bm{A}_{\mathrm{ML}}^\ast \circ \bm{A}_{\mathrm{ML}} +\sigma^2\left(\Pma\circ\Pm\right)^{-1}\right).
\end{equation}
The covariance $\bm{A}_{\mathrm{ML}}^\ast \circ \bm{A}_{\mathrm{ML}} +\sigma^2\left(\Pma\circ\Pm\right)^{-1}$ can be decomposed as $\bm{B}\circ\bm{B}^\ast$, with $\bm{B} : \mathcal{E} \rightarrow \mathcal{E} : N^{-1/2}\sum_{p=1}^q \lambda_p\bm{\epsilon}_p\bm{\varepsilon}_p^\ast + \sum_{p=q+1}^N\sigma \lambda_p^{1/2}\bm{\epsilon}_p\bm{\varepsilon}_p^\ast$ and $\left\{\bm{\varepsilon}_i\right\}_{i=1}^N$ any arbitrary orthonormal basis of the latent space $\mathcal{E}$. This decomposition allows us to sample $\bm{k}_c$ on the trained model with $\bm{k}_c = \bm{B}\bm{\xi}$ with $\bm{\xi} \sim \mathcal{N}(\bm{0},\Id{E})$. We see that $\bm{B}$ is rotational invariant, which is not surprising as this is also the case for the distribution from which $\bm{\xi}$ is sampled. In practice and for simplicity, we may decide too choose the canonical basis for $\left\{\bm{\varepsilon}_i\right\}_{i=1}^N$ as any choice would be identified to the same covariance and to the same sampling of $\bm{k}_c$. We will therefore assume that $\bm{\varepsilon}_i = \bm{e}_i$ for all $i = 1,\ldots,N$. In that particular case, $\bm{B}$ is self-adjoint and by consequence corresponds to the matrix square root of $\bm{A}_{\mathrm{ML}}^\ast \circ \bm{A}_{\mathrm{ML}} +\sigma^2\left(\Pma\circ\Pm\right)^{-1}$.

\textbf{KPCA Sampling} The classical sampling done by KPCA~\cite{schreurs} corresponds to the limit of $\sigma^2 \rightarrow 0$ for an arbitrary latent dimension $q$. Unless the latent dimension is chosen as $q=\mathrm{dim}(\mathcal{E})$, the sampling in that case can never cover $\mathcal{E}$ fully, but rather $\mathcal{L}$, as $\bm{B}$ is not a bijection. The second term of $\bm{B}$ ($\sum_{p=q+1}^N\sigma \lambda_p^{1/2}\bm{\epsilon}_p\bm{\varepsilon}_p^\ast$) allows $\bm{B}$ to be a bijection no matter what is the choice of the latent dimension $q$, as long as $\sigma^2 > 0$. We thus always sample in the full $\mathcal{E}$. This can be observed at Fig.~\ref{fig:gen}.

\begin{figure}[h]
    \centering
    \def\svgwidth{.8\columnwidth}
    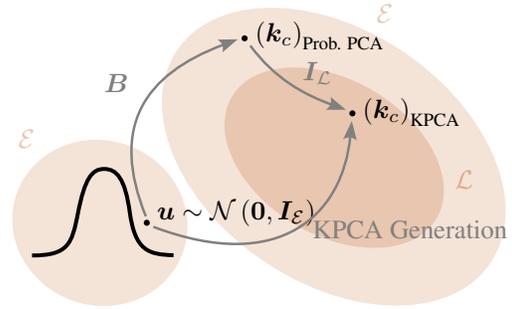
    \caption{Schematic overview of the dual sampling in Prob. PCA compared to the generation in KPCA.}
    \label{fig:gen}
\end{figure}

    
\section{Experiments}
\begin{figure}
     \centering
     \begin{subfigure}{\columnwidth}
        \vspace{-.8em}
         \centering
         \scriptsize
         \def\svgwidth{.8\columnwidth}
         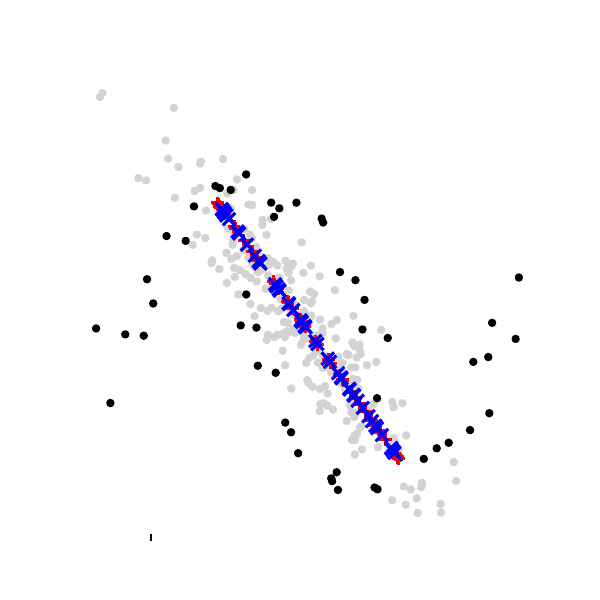
         \vspace{-1em}
         \caption{With $q=1$ component, the explained variance is $31.23\%$ and $\sigma^2 = 1.40\%$.}
         \label{fig:toy-1}
     \end{subfigure}
     \hfill
     \begin{subfigure}{\columnwidth}
         \centering
         \scriptsize
         \def\svgwidth{.8\columnwidth}
         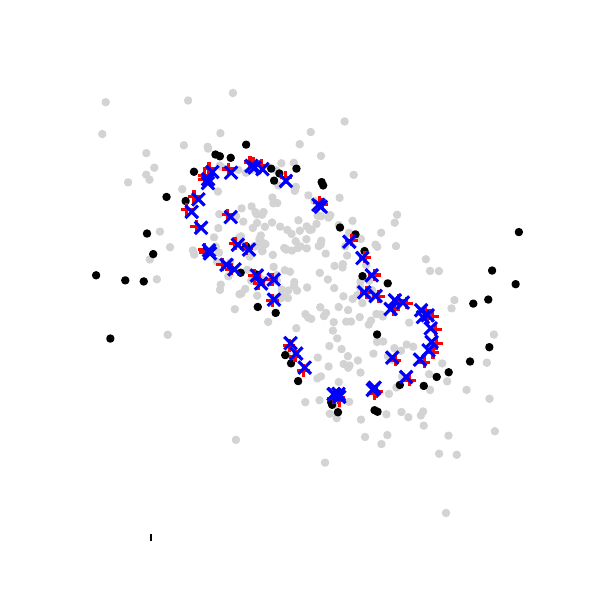
         \vspace{-1em}
         \caption{With $q=3$ components, the explained variance is $54.03\%$ and $\sigma^2 = 0.98\%$.}
         \label{fig:toy-3}
     \end{subfigure}
     \hfill
     \begin{subfigure}{\columnwidth}
         \centering
         \scriptsize
         \def\svgwidth{.8\columnwidth}
         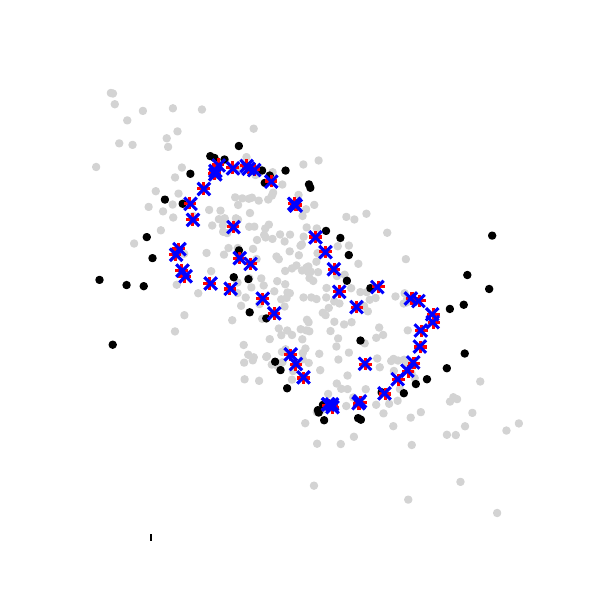
         \vspace{-1em}
         \caption{With $q=10$ components, the explained variance is $91.93\%$ and $\sigma^2 = 0.20\%$.}
         \label{fig:toy-10}
     \end{subfigure}
        \caption{Visualisation of the Probabilistic PCA reconstruction (in blue) the classical KPCA (in red). Samples generated by are also given (in grey). The dataset contains $N=20$ points (in black).}
        \label{fig:toy}
\end{figure}

\begin{figure}
     \centering
     \begin{subfigure}{\columnwidth}
        \vspace{-.8em}
         \centering
         \footnotesize
         \def\svgwidth{.8\columnwidth}
         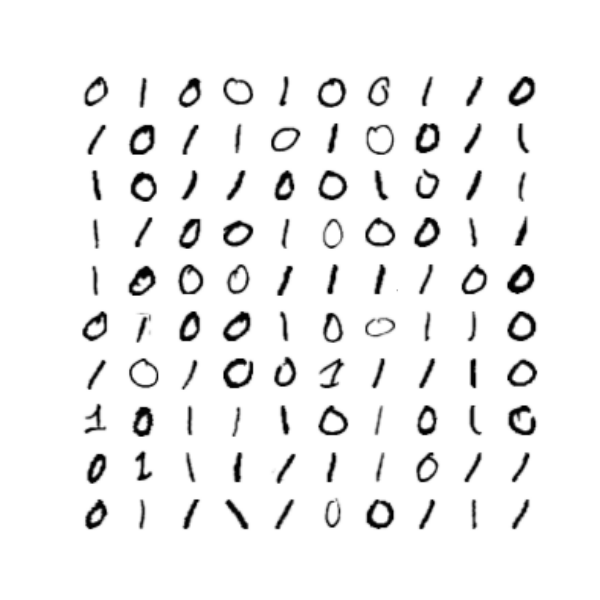
         \vspace{-1em}
         \caption{Sample of original datapoints.}
         \label{fig:mnist-original}
     \end{subfigure}
     \hfill
     \begin{subfigure}{\columnwidth}
         \centering
         \footnotesize
         \def\svgwidth{.8\columnwidth}
         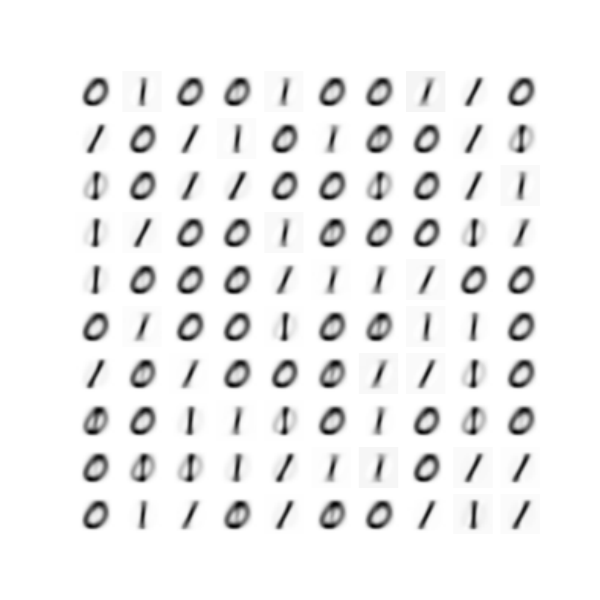
         \vspace{-1em}
         \caption{Datapoints of Fig.~\ref{fig:mnist-original} after dimensionality reduction.}
         \label{fig:mnist-recon}
     \end{subfigure}
     \hfill
     \begin{subfigure}{\columnwidth}
         \centering
         \footnotesize
         \def\svgwidth{.8\columnwidth}
         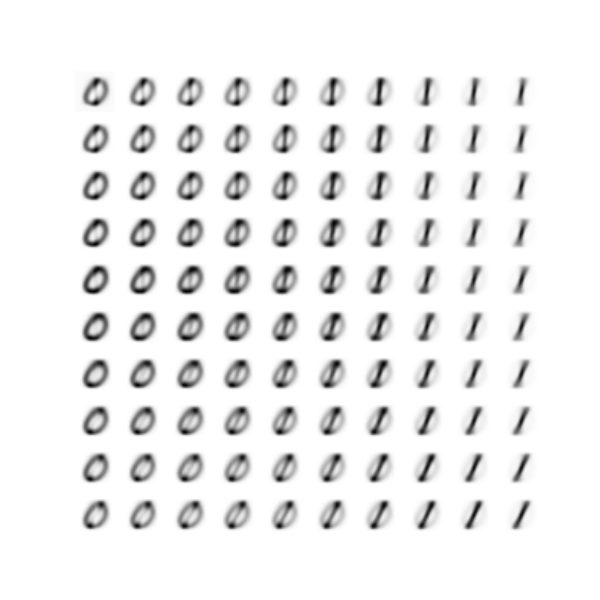
         \vspace{-1em}
         \caption{Generated datapoints. The sample $\tilde{\bm{u}}$ is uniform on $\left[-1,1\right]$ for the two first components and zero for the others. The horizontal axis varies in the first component and the vertical one in the second component.}
         \label{fig:mnist-gen}
     \end{subfigure}
        \caption{The Probabilistic PCA dual formulation on the MNIST dataset restricted to 0's and 1's, with $N=500$ datapoints, with $q=2$ components. The explained variance is $27.97\%$ and $\sigma^2=0.14\%$.}
        \label{fig:mnist}
\end{figure}
\textbf{Hilbert Spaces to Matrices.} Working in Hilbert spaces is helpful to treat possibly infinite dimensional feature maps, but not very useful for practical applications. 
Matrix representations are possible in primal if $d$ is finite and in dual if $N$ is finite. It suffices to consider the different canonical basis. For the latent space $\mathcal{L}$, this enforces a unique representation for $\bm{W}_{\mathrm{ML}}$ and $\bm{A}_{\mathrm{ML}}$, but we must keep in mind that they are rotational invariant. All the operators and elements described before are then represented in matrix or vector format (Table~\ref{tab:canonical-representations}). We will use the tilde to denote these matrices and use software-like notation by denoting with $(\cdot)_{i_1:i_2,j_1:j_2}$ the matrix truncated to its $i_1$ to $i_2$ rows and $j_1$ to $j_2$ columns. 

\textbf{Preimage.} Given a dual representation, we will also consider the \emph{kernel smoother} preimage method, as suggested by~\cite{schreurs}: 
\begin{equation}
\label{eq:preimage}
    \hat{\bm{x}} = \frac{\textstyle\sum_{i=1}^N (\tilde{\bm{k}})_i \bm{x}_i}{{\textstyle\sum_{i=1}^N (\tilde{\bm{k}})_i}}.
\end{equation}
In practice, as we work with centered feature maps and kernels, it may be that the kernel smoother may be unstable due to its normalization term. We therefore may consider to add a stabilization term.

\begin{table}[t]
    \centering
    \begin{tabular}{clll}
        \toprule 
        
        & \textbf{Name} & \textbf{Space} & \textbf{Values} \\ \midrule
       
       \parbox[t]{2mm}{\multirow{5}{*}{\rotatebox[origin=c]{90}{Dual}}}
       &$\tilde{\bm{K}}_c$ 
       & $\mathbb{R}^{N \times N}$ 
       & $(\tilde{\bm{k}}_c)_{i,j} = k_c(\bm{x}_i,\bm{x}_{j})$ \\
       
       & $\tilde{\bm{E}}$ 
       & $\mathbb{R}^{N \times N}$ 
       & $\bigl(\tilde{\bm{E}}\bigr)_{i,j} = \bm{e}_i^\ast\bm{\epsilon}_{j}$\\
       
       & $\tilde{\bm{R}}$ 
       & $\mathbb{R}^{q \times q}$ 
       & $\tilde{\bm{R}} = \bm{I}_q$\\

       & $\tilde{\bm{h}}$ 
       & $\mathbb{R}^{q}$ 
       & $\bigl(\tilde{\bm{h}}\bigr)_{p} = \bm{e}_p^\ast\bm{h}$
       \\

       & $\tilde{\bm{k}}_c$ 
       & $\mathbb{R}^{N}$ 
       & $\bigl(\tilde{\bm{k}}_c\bigr)_{i} = \bm{e}_i^\ast\bm{k}_c$
       \\

       \midrule

       \parbox[t]{2mm}{\multirow{2}{*}{\rotatebox[origin=c]{90}{Both}}}
       & $\tilde{\bm{\Lambda}}$ 
       & $\mathbb{R}^{N \times N}_{\geq 0}$ 
       & $\tilde{\bm{\Lambda}} = \mathrm{diag}(\lambda_1,\ldots,\lambda_N)$\\
       
       & $\tilde{\bm{S}}$ 
       & $\mathbb{R}_{\geq 0}^{q \times q}$ 
       & $\tilde{\bm{S}} = \mathrm{diag}(s_1,\ldots,s_q)$\\ 
       
       \midrule

       \parbox[t]{2mm}{\multirow{6}{*}{\rotatebox[origin=c]{90}{Primal}}}
        & $\tilde{\bm{C}}_c$
        & $\mathbb{R}^{d \times d}$ 
        & $\bigl(\tilde{\bm{C}}_c\bigr)_{i,j} = \bigl(\bm{u}_{i}^\ast \bm{\Phi}_c\bigr) \circ \bigl( \bm{u}_{j}^\ast \bm{\Phi}_c\bigr)^\ast$\\

        &$\tilde{\bm{\Phi}}_c$
        &$\mathbb{R}^{d\times N}$ 
        &$\bigl(\tilde{\bm{\Phi}}_c\bigr)_{i,j} = \bm{u}_i^\ast \bm{\Phi}_c \bm{e}_{j}$ \\
        
        &$\tilde{\bm{V}}$
        &$\mathbb{R}^{d\times N}$ 
        &$\bigl(\tilde{\bm{V}}\bigr)_{i,j} = \bm{u}_{i}^\ast\bm{v}_{j}$\\
        
        &$\tilde{\bm{P}}$
        &$\mathbb{R}^{d \times q}$ 
        &$\bigl(\tilde{\bm{P}}\bigr)_{i,p} = \bm{v}_i^\ast\bm{\varrho}_p$ \\ 
        
        &$\tilde{\bm{\phi}}$
        &$\mathbb{R}^{d}$ 
        &$\bigl(\bm{\phi})_{i} = \bm{v}_{i}^\ast\bm{\phi}$\\

        &$\tilde{\bm{\phi}}_c$
        &$\mathbb{R}^{d}$ 
        &$\bigl(\bm{\phi}_c)_{i} = \bm{v}_{i}^\ast\bm{\phi}_c$\\
        \bottomrule
    \end{tabular}
    \caption{Representation of the various operators and elements in their respective canonical basis, as matrices and vectors. The primal representations exist if and only if $\mathrm{dim}(\mathcal{H}) = d$ is finite.}
    \label{tab:canonical-representations}
\end{table}

%
%
%
%

\begin{table}[t]
	\centering
	\resizebox{\columnwidth}{!}{%
		\begin{tabular}{lll}
			\toprule
			
			\textbf{Name} 
			& \textbf{Space}
			& \textbf{Trained} 
			\\ \midrule
			
			$\tilde{\bm{W}}$
			& $\mathbb{R}^{d \times q}$
			& $\tilde{\bm{V}}_{1:N,1:q}\bigl(\tilde{\bm{\Lambda}}_{1:q,1:q}/N - \sigma^2\bm{I}_q \bigr)^{1/2}$
			\\
			
			$\tilde{\bm{A}}$
			& $\mathbb{R}^{N \times q}$
			& $\tilde{\bm{E}}_{1:N,1:q}\bigl(\bm{I}_q/N - \sigma^2\bigl(\tilde{\bm{\Lambda}}_{1:q,1:q}\bigr)^{-1} \bigr)^{1/2}$
			\\ 
			
			$\tilde{\bm{B}}$
			& $\mathbb{R}^{N \times q}$
			& $\tilde{\bm{E}}\tilde{\bm{\Lambda}}^{1/2}
			\left[\begin{array}{cc}
				(N)^{-1/2}\tilde{\bm{\Lambda}}_{1:q,1:q}^{1/2} & \bm{0} \\
				\bm{0} & \sigma \bm{I}_{N-q}
			\end{array}\right]$
			\\
			\bottomrule
	\end{tabular}}
	\caption{Value of the different operators in the canonical basis, after training.}
	\label{tab:trained}
\end{table}

\subsection{Model}
The direct application of the theoretical discussions of the previous sections leads to the decompositions $\tilde{\bm{K}}_c = \tilde{\bm{E}} \tilde{\bm{\Lambda}} \tilde{\bm{E}}^\top$, $\tilde{\bm{C}}_c = \tilde{\bm{V}} \tilde{\bm{\Lambda}} \tilde{\bm{V}}^\top$, $\tilde{\bm{\Phi}}_c = \tilde{\bm{V}} \tilde{\bm{\Lambda}}^{1/2}\tilde{\bm{E}}^\top$. The value of the operators after training are given in Table~\ref{tab:trained}. Once the model is trained, we can verify that $\tilde{\bm{W}} = \tilde{\bm{\Phi}}_c\tilde{\bm{A}}$.
We can also have a look at the hidden variables. A way to do it is to consider the MAP of $\bm{h}$ given $\bm{\phi}$ or $\bm{k}$. We have
\begin{align}
    \bm{h}_{\mathrm{MAP}} &= N\tilde{\bm{\Lambda}}^{-1}_{1:q,1:q}\tilde{\bm{A}}^\top\tilde{\bm{k}}_c  \,\;\quad\quad \text{(if $\mathrm{rank}(\tilde{\bm{K}}_c) \geq q$)},\\
    &= N\tilde{\bm{\Lambda}}^{-1}_{1:q,1:q}\tilde{\bm{W}}^\top\bigl(\tilde{\bm{\phi}}-\tilde{\bm{\phi}}_c\bigr) \quad \text{(if $\mathcal{H}$ is finite)},
\end{align}
and
\begin{eqnarray}
    \bigl(\bm{k}_c\bigr)_{\mathrm{MAP}} &= \tilde{\bm{K}}_c\tilde{\bm{A}} \tilde{\bm{h}}  \qquad\quad\; \text{(if $\mathrm{rank}(\tilde{\bm{K}}_c) \geq q$)},\\
    \bm{\phi}_{\mathrm{MAP}} &= \tilde{{\bm{W}}}\tilde{\bm{h}} + \tilde{\bm{\phi}}_c \qquad\qquad \text{(if $\mathcal{H}$ is finite)}.
\end{eqnarray}

As developed in Section~\ref{sec:dual-model}, we can easily generate samples in both feature and kernel representations. For the latter and in canonical basis, it becomes
\begin{eqnarray}
    \tilde{\bm{k}}_c = \tilde{\bm{B}}\tilde{\bm{u}}, \qquad \text{with $\tilde{\bm{u}} \sim \mathcal{N}(\bm{0},\bm{I}_N)$}.
\end{eqnarray}

\subsection{Examples}
As the primal case is already treated by~\cite{ppca}, we consider here the model in its dual formulation. A toy example can by found in Fig.~\ref{fig:toy}. We use an RBF kernel $k(\bm{x},\bm{y}) = \exp\bigl(-\lVert \bm{x}-\bm{y}\rVert_2^2/(2\gamma^2)\bigr)$ with bandwidth $\gamma = 2$. As the number of components increases, the mean variance of the $N-q$ unused components $\sigma^2$ becomes smaller and the model tends to the classical KPCA model. Another way the reduce $\sigma^2$ is to increase the number of components $q$, with $\sigma^2 \rightarrow 0$ when $q \rightarrow N$. This can be observed in Fig.~\ref{fig:toy-10}: the Probabilistic PCA model resembles closely the KPCA model, whereas more variance is left over, \emph{i.e.} not projected back, in Fig.s~\ref{fig:toy-1} and \ref{fig:toy-3}. The results of the generation is Gaussian, which is a consequence of the linearity of the preimage method chosen (Eq.~\eqref{eq:preimage}). Here again, as the number of components increases and $\sigma^2$ decreases, the model is allowed to project back more variance and the distribution becomes wider. Another example on the MNIST dataset~\cite{mnist} with the RBF kernel with $\gamma = 4$ is given at Fig.~\ref{fig:mnist}.
\section{Conclusion}
\textbf{Probabilistic Interpretation.} By reformulating the Prob. PCA model in Hilbert space, we were able to define a formulation of it. Likewise Prob. PCA in primal was englobing classical PCA (with $\sigma^2\rightarrow 0$), Prob. PCA in dual is also englobing KPCA in the same limit. Furthermore, we are now able to sample in dual space, enhancing the understanding of the generation done with KPCA.

\textbf{Limitations.} As most kernel methods, the model is still limited by the need of a preimage method to go back to the input space once a sample is projected or generated. Furthermore, training the model in dual required to find the $q$ first eigenvalues of the kernel matrix, which may become expensive as the number of datapoints $N$ increases. Generating renders the problem even worse as it requires the computation of all eigenvalues. The model also requires to determine a $\sigma^2$ or alternatively a latent dimension $q$.

\section*{Acknowledgements}
{\footnotesize%
	\begin{spacing}{1.04}
		EU: The research leading to these results has received funding from the European Research Council under the European Union's Horizon 2020 research and innovation program / ERC Advanced Grant E-DUALITY (787960). This paper reflects only the authors' views and the Union is not liable for any use that may be made of the contained information. Research Council KUL: Tensor Tools for Taming the Curse iBOF/23/064, Optimization frameworks for deep kernel machines C14/18/068. Flemish Government: FWO projects: GOA4917N (Deep Restricted Kernel Machines: Methods and Foundations), PhD/Postdoc grant. This research received funding from the Flemish Government (AI Research Program). Henri De Plaen and Johan A. K. Suykens are also affiliated to Leuven.AI -- KU Leuven institute for AI, B-3000, Leuven, Belgium. 
	\end{spacing}
}


\bibliography{bib}
\bibliographystyle{misc/icml2023}

\newpage
\appendix
\onecolumn
\setcounter{proposition}{0}
\setcounter{lemma}{0}
\setcounter{corollary}{0}
\setcounter{definition}{0}
\section{Theoretical Development}
\label{app:theory}
For brevity of notations, we will define the \emph{expectation value} as $\lVert \bm{a} \rVert_{\bm{\Sigma}} ^2=\bm{a}^\ast\bm{\Sigma}\bm{a} \in \mathbb{R}_{\geq 0}$, with $\bm{a}$ an element of a Hilbert space and $\bm{\Sigma}$ a linear operator from and to that space. The norm is a particular case with the identity as operator $\bm{\Sigma} = \bm{I}$. The density function of the multivariate normal distribution $\bm{a} \sim \mathcal{N}(\bm{b},\bm{\Sigma})$ can be rewritten as $\frac1Z\exp\bigl(-\frac12 \lVert \bm{a}-\bm{b}\rVert^2_{\bm{\Sigma^{-1}}}\bigr)$.

\subsection{Primal and Dual Spaces}
\label{app:covariance-kernel}
\begin{lemma}
    The operators $\PPam$ and $\PaPm$ are self-adjoint, positive semi-definite and share the same non-zero eigenvalues. In particular, we have the eigenvector relations $\bm{v}_i = (\lambda_i)^{-1/2}\bm{\Phi}_c\bm{\epsilon}_i$ and $\bm{\epsilon}_i = (\lambda_i)^{-1/2}\bm{\Phi}_c^\ast\bm{v}_i$.
\end{lemma}
\begin{proof}
    The result stated here is inspired from~\cite{convex_pca}. Self-adjointness is a consequence of the definition of the inner product, which also guarantees the positive semi-definiteness. $(\Longrightarrow)$ Let us suppose that $(\PaPm)\bm{e}_i = \lambda_i\bm{\epsilon}_i$ with $\lambda_i \neq 0$. We then have $\bm{\Phi}_c \circ (\PaPm)\bm{\epsilon}_i = (\PPam)\circ\bm{\Phi}_c\bm{\epsilon}_i = \lambda_i \bm{\Phi}_c\bm{\epsilon}_i$. Hence, we have that $\bm{\Phi}_c\bm{\epsilon}_i$ is eigenvector, but not necessarily normalized. In fact, its norm is given by $\bm{\epsilon}_i^\ast\PaPm\bm{\epsilon}_i = \lambda_i$. We thus have the relation $\bm{u}_i = (\lambda_i)^{-1/2}\bm{\Phi}\bm{\epsilon}_i$. $(\Longleftarrow)$ We suppose now that $(\PPam)\bm{u}_i$, which leads to $\bm{\Phi}_c^\ast \circ (\PPam)\bm{u}_i = (\PPa)\circ \bm{\Phi}_c^\ast\bm{u}_i = \lambda_i \bm{u}^\ast$. Again, we have the relation $\bm{\epsilon}_i = (\lambda_i)^{-1/2}\bm{\Phi}_c^\ast\bm{u}_i$.
\end{proof}
\subsection{Primal Model}
\subsubsection{Feature Distribution}
\begin{definition}[Conditional Feature Distribution]
\label{def:conditional-phi}
    Considering $\bm{\phi}, \bm{\mu} \in \mathcal{H}_{\mathcal{E}}$, $\bm{h} \in \mathcal{L}$, a linear operator $\bm{W} : \mathcal{L} \rightarrow \mathcal{H}_{\mathcal{L}}$ and its adjoint $\bm{W}^\ast : \mathcal{H}_{\mathcal{E}} \rightarrow \mathcal{L}$. We define the conditional probability distribution of the primal feature variable $\bm{\phi}$ with variance $\sigma^2 \in \mathbb{R}_{>0}$ as
\begin{equation}
\label{eq:conditional-phi}
    p\left(\bm{\phi} \left| \bm{h}\right.\right) = \frac{1}{\left(\sigma \sqrt{2\pi}\right)^{N}}\exp\left(\frac{-1}{2\sigma^2}\left\lVert\bm{\phi} - \bm{W}\bm{h} - \bm{\mu}\right\rVert_{\bm{I}_{\mathcal{H}_{\mathcal{E}}}}^2\right).
\end{equation}
\end{definition}

The following proposition verifies that the conditional distribution of $\bm{\phi}$ given $\bm{h}$ (Def.~\ref{def:conditional-phi}) is well defined. To ease the readability, we first consider a Lemma.

\label{app:cond-primal-distr}
\begin{lemma}
\label{lemma:exp-1d}
Given values $a \in \mathbb{R}_{>0}$ and $b,c \in \mathbb{R}$, we have the following integral
\begin{equation}
    \int_{\mathbb{R}} \exp\Biggl(- \frac12 ax^2+bx+c\Biggr)\mathrm{d}x = \sqrt{\frac{2\pi}{a}}\exp\left(\frac{b^2}{2a} + c\right).
\end{equation}
\end{lemma}
\begin{proof}
We first find the primitive $\int \exp\bigl(- \frac12 ax^2+bx+c\bigr)\mathrm{d}x = \sqrt{\frac{\pi}{2a}}\exp\bigl(\frac{b^2}{2a} + c\bigr)\mathrm{erf}\bigl(\frac{ax-b}{\sqrt{2a}}\bigr)$, with the error function defined as $\mathrm{erf} (x) = \frac{2}{\sqrt\pi}\int_0^x e^{-t^2}\,\mathrm dt$. Indeed, we have $\frac{\mathrm{d}}{\mathrm{d}x}\mathrm{erf}(x) = \frac{2}{\sqrt{\pi}}e^{-x^2}$. It suffices then to derivate the primitive to verify them. To conclude the proof, it suffices to notice that the error function is symmetric and that $\lim_{x\rightarrow +\infty} \mathrm{erf}(x) = 1$.
\end{proof}

From now on, we will consider the singular value decomposition of the primal interconnection linear operator $\bm{W} = \sum_{p=1}^q s_p\bm{\varrho}_p \bm{r}_p^\ast$, with $\bm{\varrho}_p \in \mathcal{H}_{\mathcal{E}}$ and $\bm{r}_p^\ast \in \mathcal{L}^\ast$ two sets of orthonormal variables and $s_p \in \mathbb{R}_{>0}$ the singular values.

\begin{proposition}
\label{prop:conditional-phi}
    Def.~\ref{def:conditional-phi} is a well-defined distribution. More specifically, its measure is normalized:
    \begin{equation}
        \int_{\mathcal{L}}  p\left(\bm{\phi} \left| \bm{h}\right.\right)\mathrm{d}\bm{\phi} = 1.
    \end{equation}
\end{proposition}
\begin{proof}
    By considering the singular value decomposition of $\bm{W}$, the term inside the exponential becomes
\begin{equation}
    -\frac{1}{2\sigma^2}\sum_{i=1}^N\Biggl\{\left(\bm{v}_i^\ast \bm{\phi}\right)^2 
    - \left(\bm{v}_i^\ast \bm{\phi}\right) \Bigl(2\sum_{p=1}^q\big(\left(\bm{r}_p^\ast \bm{h}\right)\left(\bm{\varrho}_p^\ast\bm{v}_i\right)s_p\big)+\left(\bm{v}_i^\ast \bm{\mu}\right)\Bigr)\Biggr\} + C,
\end{equation}
with $C = -\frac{1}{2\sigma^2}\left(\bm{W}\bm{h} + \bm{\mu}\right)^\ast\left(\bm{W}\bm{h} + \bm{\mu}\right) = -\frac{1}{2\sigma^2}\sum_{i=1}^N\sum_{p=1}^q \big((\bm{r}_p^\ast \bm{h})(\bm{\varrho}_p^\ast\bm{v}_i)s_p+(\bm{v}_i^\ast \bm{\mu}\big)^2$. Integrating over $\mathcal{H}_{\mathcal{E}}$ means integrating over $\mathrm{span}\left\{ \bm{v}_i,\ldots,\bm{v}_N\right\}$, thus for all $\left(\bm{v}_i^\ast \bm{\phi}\right) \in \mathbb{R}$. Using Lemma~\ref{lemma:exp-1d} and by Fubini's theorem, we have
\begin{eqnarray}
    \int_\mathcal{F} p\left(\bm{\phi} \left| \bm{h}, \bm{\varrho}\right.\right)\mathrm{d}\bm{\phi} &=&\left(\sigma\sqrt{2\pi}\right)^{-N}\exp(-C)\int_{\mathbb{R}^N}\exp\left(\sum_{i=1}^N\left(-\frac{1}{2\sigma^2}x_i^2+b_ix_i\right)\right)\mathrm{d}\bm{x}, \\
    &=&\left(\sigma\sqrt{2\pi}\right)^{-N}\exp(-C)\prod_{i=1}^N\int_{\mathbb{R}}\exp\left(-\frac{1}{2\sigma^2}x_i^2+b_ix_i\right)\mathrm{d}x_i, \\
    &=&\left(\sigma\sqrt{2\pi}\right)^{-N}\exp(-C)\prod_{i=1}^N \sigma \sqrt{2\pi}\exp\left(\frac12\sigma^2b_i^2\right), \\
    &=&\exp\left(\frac12\sigma^2\sum_{i=1}^N b_i^2-C\right), 
\end{eqnarray}
with $b_i = \frac{1}{\sigma^2}\big((\bm{r}_p^\ast \bm{h})(\bm{\varrho}_p^\ast\bm{v}_i)s_p+(\bm{v}_i^\ast \bm{\mu} \big)$. The proof is concluded by observing that $C = \frac12\sigma^2\sum_{i=1}^N b_i^2$.
\end{proof}

We can now consider the marginal distribution of the feature representation $\bm{\phi}$.
\label{app:posterior-primal-distr}
\begin{lemma}
\label{lemma:eigenvalues}
    Both linear operators $\bm{W}\circ\bm{W}^\ast + \sigma^2\bm{I}_{\mathcal{H}_{\mathcal{E}}}$ and $\bm{W}^\ast\circ\bm{W} + \sigma^2\bm{I}_{\mathcal{L}}$ are positive definite, of full rank and invertible. The linear operator $\bm{W}\circ\bm{W}^\ast + \sigma^2\bm{I}_{\mathcal{H}_{\mathcal{E}}}$ shares the $q$ non-zero eigenvalues of $\bm{W}^\ast\circ\bm{W} + \sigma^2\bm{I}_{\mathcal{L}}$, with the remaining $N-q$ eigenvalues being equal to $\sigma^2$. In particular, we have
    \begin{align}
        &\mathrm{eig}\left(\WaW + \sigma^2\Id{L}\right) = \left\{s^2_p+\sigma^2\right\}_{p=1}^q, \\
        &\mathrm{eig}\left(\WWa + \sigma^2\bm{I}_{\mathcal{H}_{\mathcal{E}}}\right) = \left\{s^2_p+\sigma^2\right\}_{p=1}^q\cup \left\{\sigma^2\right\}^{N-q}.
    \end{align}
\end{lemma}
\begin{proof}
    
    We consider the eigenvalues of $\bm{W}^\ast\circ\bm{W} + \bm{I}_{\mathcal{L}} = \sum_{p=1}^q (s_p^2+\sigma^2)\bm{\varrho}_p\bm{\varrho}_p^\ast$. Its eigenvalues are thus given by $s_p^2 + \sigma^2$ for $p = 1,\ldots,q$ and we directly conclude that they are all strictly positive. In a similar fashion, the eigenvalues of $\bm{W}\circ\bm{W}^\ast + \sigma^2\bm{I}_{\mathcal{H}_{\mathcal{E}}} = \sum_{p=1}^q(s_p^2 + \sigma^2)\bm{r}_p\bm{r}_p^\ast + \sigma^2\PkW$ are given by $s_p^2+\sigma^2$ for $p=1,\ldots,q$ and $\sigma^2$ for $k = q+1,\ldots,N$. All the eigenvalues are strictly positive. Hence are the operators positive definite, full rank and invertible.
\end{proof}

\begin{proposition}[Marginal Feature Distribution]
\label{prop:posterior-phi}
    Assuming the conditional distribution of $\bm{\phi}$ given $\bm{h}$ (Def.~\ref{def:conditional-phi}), assumed normally distributed, \emph{i.e.}, $p(\bm{h}) = \left(2\pi\right)^{-q/2}\exp\left(-\frac12 \bm{h}^\ast\bm{h}\right)$, the posterior distribution of the primal feature vector $\bm{\phi}$ is given by
    \begin{equation}
        p(\bm{\phi}) = \frac{1}{Z_{\bm{\phi}}} \exp\left(-\frac12\left\lVert\bm{\phi}-\bm{\mu}\right\rVert^2_{\bm{\Sigma}_{\bm{\phi}|\bm{W}}^{-1}}\right),
    \end{equation}
with $Z_{\bm{\phi}} = (2\pi)^{N/2}\bigl((\sigma^2)^{N-q}\prod_{p=1}^q(s_p^2+\sigma^2)\bigr)^{1/2}$ and $\bm{\Sigma}_{\bm{\phi}|\bm{W}} = \bm{W}\circ\bm{W}^\ast + \sigma^2\bm{I}_{\mathcal{H}_{\mathcal{E}}}$
\end{proposition}
\begin{proof}
    The joint probability distribution is given by Bayes' theorem as $p(\bm{\phi},\bm{h}) = p(\bm{\phi}|\bm{h})p(\bm{h})$. The integration proceeds in a very similar fashion as the proof of Prop.~\ref{prop:conditional-phi}.
Similarly, the term inside the exponential becomes
\begin{equation}
        -\frac{1}{2\sigma^2}\sum_{p=1}^q\Biggl\{\left(\bm{r}_p^\ast \bm{h}\right)^2\left(s_p^2 + \sigma^2\right)
        -\left(\bm{r}_p^\ast\bm{h}\right)\Bigl(2s_p\bigl(\bm{\varrho}_p^\ast (\bm{\phi} - \bm{\mu})\bigl)\Bigr)\Biggr\} + D,
\end{equation}
with 
\begin{eqnarray}
    D &=& -\frac{1}{2\sigma^2}\left(\bm{\phi} - \bm{\mu}\right)^\ast\left(\bm{\phi} - \bm{\mu}\right), \\
    &=& -\frac{1}{2\sigma^2}\sum_{p=1}^q(\bm{\varrho}_p^\ast(\bm{\phi}-\bm{\mu}))^2 - \frac{1}{2\sigma^2}\left(\bm{\phi}- \bm{\mu}\right)^\ast\PkW\left(\bm{\phi} - \bm{\mu}\right).
\end{eqnarray} 

Here again, integrating on the whole latent space $\mathcal{L}$ means integrating for all $\left(\bm{r}_p^\ast\bm{h}\right) \in \mathbb{R}$. Again, we use Lemma~\ref{lemma:exp-1d} and Fubini's theorem. After some simplification inside the exponential, we can use the development made in Lemma~\ref{lemma:eigenvalues} to find the following:
\begin{equation}
    -\frac12\left\{\sum_{p=1}^q\frac{(\bm{\varrho}_p^\ast(\bm{\phi}-\bm{\mu}))^2}{s_p^2+\sigma^2} + \frac{1}{\sigma^2}\left(\bm{\phi}- \bm{\mu}\right)^\ast\PkW\left(\bm{\phi} - \bm{\mu}\right)\right\}
    = -\frac12 \left(\bm{\phi}- \bm{\mu}\right)^\ast\left(\WWa + \sigma^2\Id{\mathcal{H}_{\mathcal{E}}}\right)^{-1}\left(\bm{\phi} - \bm{\mu}\right).
\end{equation}
The normalization term follows similarly to Proposition~\ref{prop:conditional-phi} and we can verify the consistency of the obtained distribution by looking at the spectrum using Lemma~\ref{lemma:eigenvalues}.
\end{proof}

\subsubsection{Maximum Likelihood}
Training the model in primal corresponds to maximizing the likelihood of the observations in the finite feature space $\mathcal{H}_\mathcal{E}$.
\label{app:mle}
\begin{proposition}[Primal Maximum Likelihood]
\label{prop:mle-primal}
    Provided $\sigma^2 \leq \lambda_{q}/N$, where $\lambda_{q}$ is the $q$\textsuperscript{th} greatest 
    eigenvalue of $\PPam$, the \emph{Maximum Likelihood (ML)} of $\bm{W}$ and $\bm{\mu}$ given the observations $\{\bm{\phi}_i\}_{i=1}^N$ is given by
    \begin{eqnarray}
        \bm{\mu}_{\mathrm{ML}} &=& \bm{\phi}_c,\\
        \bm{W}_{\mathrm{ML}} &=& \sum_{p=1}^q \sqrt{\lambda_p/N-\sigma^2}\bm{v}_p\bm{r}_p^\ast,
    \end{eqnarray}
with $\left\{\left(\lambda_p,\bm{v}_p\right)\right\}_{p=1}^q$ the greatest eigenpairs of $\PPam$ (w.r.t. the eigenvalues).
\end{proposition}
\begin{proof}
    The maximum likelihood of the observations $\{\bm{\phi}_i\}_{i=1}^N$ is computed as $\mathrm{argmax}_{\bm{W},\bm{\mu}} \log \bigl(\prod_{i=1}^N p(\bm{\phi}_i|\bm{W},\bm{\mu})\bigr) = \mathrm{argmax}_{\bm{W},\bm{\mu}} L_{\bm\phi}$, with $L_{\bm{\phi}}$ the likelihood function, which can be written as 
    \begin{equation}
    \begin{split}
        L_{\bm{\phi}} = -\frac{N}{2}\Biggl\{&N\log(2\pi) + (N-q)\log\left(\sigma^2\right) + \sum_{p=1}^q\log\left(s_p^2+\sigma^2\right) + \frac{1}{N\sigma^2}\sum_{i=1}^N(\bm{\phi}_i - \bm{\mu})^\ast\PkW(\bm{\phi}_i- \bm{\mu}) \\
        &+ \frac1N\sum_{p=1}^q \frac{1}{s_p^2+\sigma^2}\left(\sum_{i=1}^N \left(\bm{\varrho}_p^\ast(\bm{\phi}_i- \bm{\mu})\right)\right)^2\Bigg\}.
    \end{split}
    \end{equation}
    The optimization of the mean $\bm{\mu}$ is trivial and we have $\bm{\mu}_{\mathrm{ML}} = \bm{\phi}_c$. The optimization of $\bm{W} = \sum_{p=1}^q s_p \bm{\varrho}_p\bm{r}_p^\ast$ is less trivial. We first note that optimizing for $\{\bm{\varrho}_p\}_{p=1}^q$, $\{s_p^2\}_{p=1}^q$ and $\{\bm{r}_p\}_{p=1}^q$ is not identifiable: indeed, two singular values may be equal. Furthermore, the likelihood function $L_{\bm{\phi}}$ is independent of the basis $\{\bm{r}_p\}_{p=1}^q$ and for any solution of $\bm{\varrho}_p$, we must also admit its opposite $-\bm{\varrho}_p$ too. In addition to that, computing the saddle points of the likelihood is not straightforward as we cannot consider the vectors $\{\bm{\varrho}_p\}_{p=1}^q$ to be independent variables as they must remain orthonormal: the variation of one basis vector $\bm{\varrho}_p$ must keep it normalized and has an influence on the other basis vectors. For $s_p$ however, the variations may happen with the only constraint of strict positivity. We may thus consider
    \begin{equation}
         \frac{\partial L_{\bm{\phi}}}{\partial s_p} = 0 \Longleftrightarrow s_p^2+\sigma^2 = \frac1N\left(\bm{\varrho}_p^\ast\Bigl(\sum_{n=1}^N\bm{\phi}_n- \bm{\mu}\Bigr)\right)^2.\label{eq:partial-s}
    \end{equation}

    We now consider the fact that $\PkW = \Id{\mathcal{H}_{\mathcal{E}}} - \sum_{p=1}^q \bm{\varrho}_p\bm{\varrho}_p^\ast$. The likelihood function restricted to the terms dependent on $\bm{\varrho}_p$ reduces to
    \begin{equation}
        \frac12\left\{\sum_{p=1}^q\left(\frac{1}{s_p^2+\sigma^2} - \frac{1}{\sigma^2}\right)\bm{\Phi}_c^\ast \circ \left(\bm{\varrho}_p\bm{\varrho}_p^\ast\right) \circ \bm{\Phi}_c\right\}.
    \end{equation}

    From there, we can deduce that the likelihood is maximized when $\{\bm{\varrho}_p\}_{p=1}^q$ forms a basis of $\mathrm{span}(\bm{v}_1,\ldots,\bm{v}_p)$, with $\{v_p\}_{p=1}^q$ the $q$ greatest eigenvectors of $\PPam$ (w.r.t. the eigenvalues $\lambda_p$'s). We may thus identify both basis: $\bm{\varrho}_p = \bm{v}_p$ and by consequence $s_p^2 + \sigma^2 = \lambda_p/N$ (from Eq.~\eqref{eq:partial-s}), for all $p=1,\ldots,q$. 
    
    What about the other choices of basis? At the end, it would not change anything as it would be identified to the same solution. Indeed, as the $\bm{\varrho}_p$'s could always be written as a linear combination of $\{\bm{v}_p\}_{p=1}^q$, we could always write the interconnection operator as $\bm{W}_{\mathrm{ML}} = \sum_{p=1}^q (\lambda_p/N - \sigma^2) \bm{v}_p\bm{r}_p^\ast$ as the choice of the orthonormal basis $\{\bm{r}_p\}_{p=1}^q$ is arbitrary.
\end{proof}

\subsubsection{Dimensionality Reduction in Feature Space}
\begin{proposition}[Primal Conditional Latent Distribution]
\label{prop:conditional-latent-primal}
    The conditional distribution of $\bm{h}$ given $\bm{\phi}$ is given by
    \begin{equation}
        p\left(\bm{h}\left|\bm{\phi}\right.\right) = \frac{1}{Z_{\bm{h}|\bm{\phi}}}
        \exp\left(-\frac{1}{2\sigma^2}\left\lVert\bm{h} -\bm{m}\right\rVert^2_{\bm{\Sigma}^{-1}_{\bm{h}|\bm{\phi}}}\right),
    \end{equation}
with $\bm{\Sigma}_{\bm{h}|\bm{\phi}} = \left(\WaW + \sigma^2\Id{\mathcal{L}}\right)^{-1}$, $\bm{m} = (\bm{\Sigma}_{\bm{h}|\bm{\phi}})^{-1}\bm{W}^\ast(\bm{\phi}-\bm{\mu})$ and $Z_{\bm{h}|\bm{\phi}} = (\sigma\sqrt{2\pi})^{q}\bigl(\prod_{p=1}^q(s_p^2+\sigma^2)\bigr)^{-1/2}$.
\end{proposition}
\begin{proof}
    The methodology is analogous to Props.~\ref{prop:conditional-phi} and \ref{prop:posterior-phi}.
\end{proof}
\subsection{Dual Model}

\subsubsection{Representation}
\label{app:representation}
\begin{proposition}[Dual Representation]
    Given any interconnection operator $\bm{W} : \mathcal{L} \rightarrow \mathcal{H}_{\mathcal{L}}$, we have the following representation:
    \begin{equation}
        \bm{W} = \bm{\Phi}_c \circ \bm{A},
    \end{equation}
    with $\bm{A} : \mathcal{L} \rightarrow \mathcal{L}$, the dual interconnection operator.
\end{proposition}
\begin{proof}
    Following our definitions, we know that $\mathrm{\PPam} \geq q$. By consequence, the linear operator $\bm{\Phi}_c : \mathcal{H} \rightarrow \mathcal{E}$ has at least $q$ non-zero singular values. As we have $\mathrm{dim}(\mathcal{L}) = \mathrm{dim}(\mathcal{H}_{\mathcal{L}})$, the proof is concluded by recalling that the primal interconnection operator $\bm{W} : \mathcal{L} \rightarrow \mathcal{H}_{\mathcal{L}}$ is also linear.
\end{proof}

\begin{proposition}[Dual Maximum Likelihood]
\label{prop:dual-mle}
    Given the operator $\bm{W}_{\mathrm{ML}}$ (Prop.~\ref{prop:mle-primal}) and provided $\sigma^2 \leq \lambda_{q}/N$, where $\lambda_{q}$ is the $q$\textsuperscript{th} greatest eigenvalue of $\PaPm$, the dual interconnection operator $\bm{A}_{\mathrm{ML}}: \mathcal{L} \rightarrow \mathcal{L}$ is given by:
    \begin{equation}
        \bm{A}_{\mathrm{ML}} = \sum_{p=1}^q \sqrt{1/N-\sigma^2\lambda_p^{-1}}\bm{\epsilon}_p\bm{r}_p^\ast,
    \end{equation}
    with $\left\{\left(\lambda_p,\bm{\epsilon}_p\right)\right\}_{p=1}^q$ the greatest eigenpairs of $\PaPm$ (w.r.t. the eigenvalues) and $\left\{\bm{r}_p\right\}_{p=1}^q$ an arbitrary orthonormal basis of the latent space $\mathcal{L}$.
\end{proposition}
\begin{proof}
    Using the relation $\bm{v}_p = (\lambda_p)^{-1/2}\bm{\Phi}_c \bm{\epsilon}_p$, we have
    \begin{eqnarray}
        \bm{W}_{\mathrm{ML}} &=& \sum_{p=1}^q \sqrt{\lambda_p/N-\sigma^2}\bm{v}_p\bm{r}_p^\ast, \\
        &=& \sum_{p=1}^q (\lambda_p)^{-1/2}\sqrt{\lambda_p/N-\sigma^2}\bm{\Phi}_c \bm{\epsilon}_p\bm{r}_p^\ast, \\
        &=& \bm{\Phi}_c \circ \left(\sum_{p=1}^q \sqrt{1/N-\sigma^2\lambda_p^{-1}}\bm{\epsilon}_p\bm{r}_p^\ast\right).
    \end{eqnarray}
    The proof is concluded by insuring that $1/N-\sigma^2\lambda_p^{-1}$ is never negative.
\end{proof}

\subsubsection{Dimensionality Reduction in Kernel Space}
\label{app:dual-red}
\begin{proposition}[Conditional Kernel Distribution]
    Under the same assumption as Prop.~\ref{prop:dual-mle} and provided that $\PaPm$ is of full rank, the posterior distribution of the dual feature variable $\bm{k}_c$ given the latent variable $h$ is given by
    \begin{equation}
         p\left(\bm{k}_c|\bm{h}\right) = \frac{1}{Z_{\bm{k}_c|\bm{h}}}\exp\left(-\frac{1}{2\sigma^2}
            \left\lVert \bm{k}_c - \left( \PaPm\right)\circ \bm{A}_{\mathrm{ML}}\bm{h}\right\rVert^2_{\bm{\Sigma}^{-1}_{\bm{k}_c|\bm{h}}}
        \right),
    \end{equation}
    with $Z_{\bm{k}_c|\bm{h}} = \left(\sigma\sqrt{2\pi}\right)^N\prod_{i=1}^N\lambda_i$ and $\bm{\Sigma}_{\bm{k}_c|\bm{h}} = \PaPm$.
\end{proposition}
\begin{proof}
    We consider the optimal primal interconnection operator $\bm{W}_{\mathrm{ML}} = \bm{\Phi}_c \circ \bm{A}_{\mathrm{ML}}$ together with the optimal mean $\bm{\mu}_{\mathrm{ML}} = \bm{\phi}_c$ and fill it in Def.~\ref{def:conditional-phi}. The covariance is deducted by considering $\bm{\phi}-\bm{\phi}_c = \bigl(\bm{\Phi}^\ast_c\bigr)^{-1}\circ \bm{\Phi}^\ast_c(\bm{\phi}-\bm{\phi}_c) = \bigl(\bm{\Phi}^\ast_c\bigr)^{-1}\bm{k}_c$.
\end{proof}

\begin{proposition}[Dual Conditional Latent Distribution]
\label{prop:conditional-latent-dual}
Under the same assumption as Prop.~\ref{prop:dual-mle}, the posterior distribution of the latent variable $\bm{h}$ given the dual feature variable $\bm{k}_c$ is given by
     \begin{equation}
     \begin{split}
         p(\bm{h}|\bm{k}_c) = \frac{1}{Z_{\bm{h}|\bm{k}_c}}\exp\left(
         -\frac{1}{2\sigma^2}\left\lVert\bm{h} - \bm{m}'\right\rVert^2_{\bm{\Sigma}^{-1}_{\bm{h}|\bm{k}_c}}
         \right),
     \end{split}
     \end{equation}
 with $\bm{\Sigma}_{\bm{h}|\bm{k}_c} = \left(\bm{A}_{\mathrm{ML}}^\ast \circ \bigl(\PaPm\bigr) \circ \bm{A}_{\mathrm{ML}} + \sigma^2 \bm{I}_{\mathcal{L}}\right)^{-1}$, $Z_{\bm{h}|\bm{k}_c} = Z_{\bm{h}|\bm{\phi}}$ (Prop.~\ref{prop:conditional-latent-primal}) and $\bm{m}' = \bigl(\bm{\Sigma}_{\bm{h}|\bm{k}_c}^{-1}\circ \bm{A}_{\mathrm{ML}}^\ast \bigr)\bm{k}_c$.
 \end{proposition}
 \begin{proof}
     It suffices to consider the optimal primal interconnection operator $\bm{W}_{\mathrm{ML}} = \bm{\Phi}_c \circ \bm{A}_{\mathrm{ML}}$, as well as the optimal mean $\bm{\mu}_{\mathrm{ML}} = \bm{\phi}_c$ and fill both in Prop.~\ref{prop:conditional-latent-primal}.
 \end{proof}

\subsubsection{Kernel Sampling}
\label{app:dual-distr}

\begin{proposition}[Marginal Kernel Distribution]
\label{prop:dual-posterior}
    Provided that $\PaPm$ is of full rank 
    and given the assumptions of Prop.~\ref{prop:dual-mle}, the posterior distribution of the the dual feature $\bm{k}_c$ after optimization is given by
    \begin{equation}
        p(\bm{k}_c) = \frac{1}{Z_{\bm{k}_c}}\exp\left(-\frac12\left\lVert\bm{k}_c\right\rVert^2_{\bm{\Sigma}_{\bm{k}_c}^{-1}}\right),
    \end{equation}
    with $\bm{\Sigma}_{\bm{k}_c} = \bm{A}_{\mathrm{ML}}^\ast \circ \bm{A}_{\mathrm{ML}} +\sigma^2\left(\Pma\circ\Pm\right)^{-1}$ and $Z_{\bm{k}_c} = (2\pi)^{N/2}N^{-q/2}\prod_{p=1}^q \lambda_p\prod_{p=q+1}^N(\sigma^2\lambda_p)^{1/2}$.
\end{proposition}
\begin{proof} We start from Prop.~\ref{prop:posterior-phi} and develop the covariance: 
    \begin{eqnarray}
        \bm{W}_{\mathrm{ML}}\bm{W}_{\mathrm{ML}}^\ast + \sigma^2\bm{I}_{\mathcal{H}_{\mathcal{E}}} &=&\Pm \circ \left(\bm{A}^\ast \circ \bm{A}+\sigma^2\left(\Pma\circ\Pm\right)^{-1}\right)\circ \Pma,\\
        &=&\Pm\circ\left(\frac1N\sum_{p=1}^q\bm{\epsilon}_p\bm{\epsilon}_p^\ast+\sigma^2\!\!\sum_{p=q+1}^N\!\lambda_p^{-1}\bm{\epsilon}_p\bm{\epsilon}_p^\ast\right)\circ\Pma.
    \end{eqnarray}
    We set $\bm{\mu}_{\mathrm{ML}} = \bm{\phi}_c$ and observe that 
    \begin{eqnarray}
        \left(\bm{\phi} - \bm{\phi}_c\right)^\ast\left(\bm{W}_{\mathrm{ML}}\bm{W}_{\mathrm{ML}}^\ast+\sigma^2\bm{I}_{\mathcal{H}_{\mathcal{E}}}\right)^{-1}\left(\bm{\phi} - \bm{\phi}_c\right),
        &=&\bm{k}_c^\ast\left(\bm{\Phi}_c^\ast\circ\left(\bm{W}_{\mathrm{ML}}\bm{W}_{\mathrm{ML}}^\ast +  \sigma^2\bm{I}_{\mathcal{H}_{\mathcal{E}}}\right)\circ\bm{\Phi}_c\right)^{-1}\bm{k}_c, \\
        &=&\bm{k}_c^\ast\left(\frac1N\sum_{p=1}^q \lambda_p^2\bm{\epsilon}_p\bm{\epsilon}_p^\ast+\sigma^2\!\!\sum_{p=q+1}^N\!\lambda_p\bm{\epsilon}_p\bm{\epsilon}_p^\ast\right)\bm{k}_c.
    \end{eqnarray}
\end{proof}



\end{document}